%% file: NonAsympSA-TD-arXiv.tex
\documentclass[11pt,twoside]{article}

\usepackage{fullpage}
\usepackage{epsf}
\usepackage{fancyhdr}
\usepackage{graphics}
\usepackage{graphicx}
\usepackage{psfrag}
\usepackage{caption}
\usepackage{subcaption}
\usepackage{color}
\usepackage{amssymb}
\usepackage{amsthm}
\usepackage{amsfonts}
\usepackage{amsmath}
\usepackage{amssymb}
\usepackage{hyperref}
\usepackage{algorithm}
\usepackage{algorithmic}
\usepackage{bbm}
\usepackage{tikz}
\usepackage{caption}

\usepackage{mathrsfs}

\usepackage{tikz,lipsum,lmodern}
\usepackage[most]{tcolorbox}

\usepackage{macros-arXiv}
\input{BearMacros-arXiv}

\newcommand{\Dim}{\ensuremath{D}}
\newcommand{\rtil}{\ensuremath{\widetilde{r}}}

\usepackage{array}
\newcolumntype{P}[1]{>{\centering\arraybackslash}p{#1}}

\newcommand{\problem}{\mathcal{P}}

\newcommand{\tv}{{\rm TV}}

\newcommand{\obsX}{\mathbf{Z}}

\newcommand{\obsr}{R}

\newcommand{\remainder}{\mathcal{R}}
\newcommand{\problemclass}{\mathcal{S}}
\newcommand{\one}{\mathbf{1}}
\newcommand{\Umat}{\ensuremath{\mathbf{U}}}

\newcommand{\Zsamp}{\mathbf{Z}}

\newcommand{\discount}{\ensuremath{\gamma}}
\newcommand{\Nsamp}{\nobs}
\newcommand{\lbmaxrew}{\nu}
\newcommand{\betastar}{\ensuremath{\beta^*}}

\newcommand{\reward}{\ensuremath{r}}
\newcommand{\qproblem}{\ensuremath{\mathcal{Q}}}

\newcommand{\tdstep}{\ensuremath{\alpha}}
\newcommand{\parMRP}{\ensuremath{\lambda}}

\newcommand{\PmatPrime}{\ensuremath{\Pmat^\prime}}

\newcommand{\hardProblem}{\widebar{\problem}}
\newcommand{\hardPmat}{\bar{\Pmat}}
\newcommand{\tauOne}{\ensuremath{\tau}}
\newcommand{\tauTwo}{\ensuremath{\eta}}
\newcommand{\DeltaPPrime}{\ensuremath{\Delta_{\Pmat}}}

\newcommand{\MatDiff}[1]{\ensuremath{M_{#1}}}

\newcommand{\NumEpoch}{\ensuremath{\TotalEpochs}}
\newcommand{\epiter}{\ensuremath{m}}

\newcommand{\neighbor}{\mathfrak{N}}
\newcommand{\p}{\ensuremath{p}}

\newcommand{\Aux}{\ensuremath{V}}

\newcommand{\vartarget}{\ensuremath{\vartheta}}
\newcommand{\varthetaprime}{\ensuremath{\vartheta'}}

\newcommand{\betahat}{\ensuremath{\hat{\beta}}}

\newcommand{\Fisher}{\ensuremath{J}}

\newcommand{\Amat}{\ensuremath{\mathbf{A}}}

\begin{document}


\begin{center}

{\bf{\LARGE{\mbox{Is Temporal Difference Learning Optimal?} \\[.1cm]
      An Instance-Dependent Analysis}}}

\vspace*{.2in}

{\large{
 \begin{tabular}{ccc}
  Koulik Khamaru$^{\diamond}$ & Ashwin Pananjady$^{\dagger}$ &  Feng Ruan$^{\dagger}$ \\
 \end{tabular}
 \begin{tabular}
 {ccc}
  Martin J. Wainwright$^{\diamond, \dagger, \ddagger}$ & Michael I. Jordan$^{\diamond, \dagger}$ \textbf{}
 \end{tabular}
}}

 \vspace*{.2in}

 \begin{tabular}{c}
 Department of Statistics$^\diamond$ \\
 Department of Electrical Engineering and Computer Sciences$^\dagger$\\
 UC Berkeley\\
 \end{tabular}

 \vspace*{.1in}
 \begin{tabular}{c}
 The Voleon Group$^\ddagger$
 \end{tabular}
\vspace*{.2in}
\textbf{}
\vspace*{.2in}

\today

\vspace*{.2in}

\begin{abstract}
We address the problem of policy evaluation in discounted Markov decision
processes, and provide instance-dependent guarantees on the $\ell_\infty$-error
under a generative model.  We establish both asymptotic and
non-asymptotic versions of local minimax lower bounds for policy
evaluation, thereby providing an instance-dependent baseline by which
to compare algorithms. Theory-inspired simulations show that the widely-used
temporal difference (TD) algorithm is strictly suboptimal when evaluated
in a non-asymptotic setting, even when combined with Polyak-Ruppert
iterate averaging.  We remedy this issue by introducing and analyzing
variance-reduced forms of stochastic approximation, showing that they
achieve non-asymptotic, instance-dependent optimality up to
logarithmic factors.
\end{abstract}

\end{center}

\section{Introduction}
\label{sec:intro}

Reinforcement learning (RL) is a class of methods for the optimal control of dynamical
systems~\cite{Bertsekas_dyn1,Bertsekas_dyn2,SutBar18,bertsekas1996neuro} that has begun 
to make inroads in a wide range of applied problem domains.  This empirical research has 
revealed the limitations of our theoretical understanding of this class of methods---popular 
RL algorithms exhibit a variety of behavior across domains and problem instances, and 
existing theoretical bounds, which are generally based on worst-case assumptions, fail to 
capture this variety. An important theoretical goal is to develop \emph{instance-specific} 
analyses that help to reveal what aspects of a given problem make it ``easy'' or ``hard,'' and 
allow distinctions to be drawn between ostensibly similar algorithms in terms of their 
performance profiles.  The focus of this paper is on developing such a theoretical 
understanding for a class of popular stochastic approximation algorithms used for 
policy evaluation.

RL methods are generally formulated in terms of a Markov decision process (MDP). 
An agent operates in an environment whose dynamics are described by an MDP but 
are unknown: at each step, it observes the current
state of the environment, and takes an action that changes the state
according to some stochastic transition function. The eventual goal of the agent is to 
learn a \emph{policy}---a mapping from states to actions---that optimizes the reward 
accrued over time. In the typical setting, rewards are assumed to be additive over time, 
and are also discounted over time. Within this broad context, a key sub-problem is 
that of \emph{policy evaluation}, where the goal is estimate the long-term expected 
reward of a fixed policy based on observed state-to-state transitions and one-step rewards. 
It is often preferable to have $\ell_\infty$-norm guarantees for such an estimate, 
since these are particularly compatible with policy-iteration methods.
In particular, policy iteration can be shown to converge at a geometric
rate when combined with policy evaluation methods that are accurate in
$\ell_\infty$-norm~(see, e.g.,\cite{agarwal2019reinforcement,bertsekas1996neuro}).

In this paper, we study a class of stochastic approximation algorithms
for this problem under a generative model for the underlying MDP, with
a focus on developing instance-dependent bounds. Our results
complement an earlier paper by a subset of the
authors~\cite{pananjady2019value}, which studied the least squares
temporal difference (LSTD) method through such a lens.


\subsection{Related work}

We begin with a broad overview of related work, categorizing that work as involving asymptotic analysis, non-asymptotic analysis, or instance-dependent analysis.

\paragraph{Asymptotic theory:}

Markov reward processes have been the subject of
considerable classical study~\cite{Feller2,durrett1999essentials}.  
In the context of reinforcement learning and stochastic control, 
the policy evaluation problem for such processes has been tackled by various approaches based on
stochastic approximation.  Here we focus on past work that studies 
the \emph{temporal difference} (TD) update and its relatives; see~\cite{dann2014policy} 
for a comprehensive survey.  The TD update was originally proposed by 
Sutton~\cite{sutton1988learning}, and is typically used in conjunction with an 
appropriate parameterization of
value functions.  Classical results on the algorithm are typically
asymptotic, and include both convergence 
guarantees~\cite{jaakkola1994convergence,borkar2009stochastic,borkar2000ode}
and examples of divergence~\cite{baird1995residual}; see the seminal
work~\cite{tsitsiklis1997analysis} for
conditions that guarantee asymptotic convergence.

It is worth noting that the TD algorithm is a form of linear
stochastic approximation, and can be fruitfully combined with the
iterate-averaging procedure put forth independently by
Polyak~\cite{polyak1992acceleration} and Ruppert~\cite{Rup88}.  In
this context, the work of Polyak and
Juditsky~\cite{polyak1992acceleration} deserves special mention, since
it shows that under fairly mild conditions, the TD algorithm converges
when combined with Polyak-Ruppert iterate averaging.  To be clear, in
the specific context of the policy evaluation problem, the results in the
Polyak-Juditsky paper~\cite{polyak1992acceleration} allow noise only
in the observations of rewards (i.e., the transition function is
assumed to be known).  However, the underlying techniques can be
extended to derive results in the setting in which we only observe 
samples of transitions; for instance, see the work of 
Tadic~\cite{tadic2004almost} for results of this type.


\paragraph{Non-asymptotic theory:}
Recent years have witnessed significant interest in understanding TD-type
algorithms from the non-asymptotic standpoint.  Bhandari et
al.~\cite{bhandari2018finite} focus on proving $\ell_2$-guarantees for
the TD algorithm when combined with Polyak-Ruppert iterate averaging.
They consider both the generative model as well as the Markovian noise
model, and provide non-asymptotic guarantees on the expected
error. Their results also extend to analyses of the popular
TD($\lambda$) variant of the algorithm, as well as to $Q$-learning in
specific MDP instances.  Also noteworthy is the analysis of
Lakshminarayanan and Szepesvari~\cite{lakshminarayanan2018linear},
carried out in parallel with Bhandari et al.~\cite{bhandari2018finite};
it provides similar guarantees on the TD$(0)$ algorithm with constant
stepsize and averaging.  Note that both of these analyses focus on
$\ell_2$-guarantees (equipped with an associated inner product), and
thus can directly leverage proof techniques for stochastic
optimization~\cite{BacMou11,NemJudLanSha09}.

Other related results\footnote{There were some errors in the results
  of Korda and La~\cite{korda2015td} that were pointed out by both
  Lakshminarayanan and Szepesvari~\cite{lakshminarayanan2018linear}
  and Xu et al.~\cite{xu2020reanalysis}.} include those of Dalal et
al.~\cite{dalal2018finite}, Doan et al.~\cite{doan2019finite}, Korda
and La~\cite{korda2015td}, and also more contemporary papers~\cite{xu2020reanalysis,wai2019variance}. The latter three of these papers introduce a
variance-reduced form of temporal difference learning, a variant of
which we analyze in this paper.

\paragraph{Instance-dependent results:}

The focus on instance-dependent guarantees for TD algorithms is
recent, and results are available both in the $\ell_2$-norm
setting~\cite{bhandari2018finite,lakshminarayanan2018linear,dalal2018finite,
  xu2020reanalysis} and the $\ell_\infty$-norm
settings~\cite{pananjady2019value}.  In general, however, the guarantees
provided by work to date are not sharp.  For instance, the bounds in~\cite{dalal2018finite} scale exponentially in relevant
parameters of the problem, whereas the
papers~\cite{bhandari2018finite,lakshminarayanan2018linear,xu2020reanalysis}
do not capture the correct ``variance'' of the problem instance at
hand. A subset of the current authors~\cite{pananjady2019value}
derived $\ell_\infty$ bounds on policy evaluation for the plug-in
estimator.  These results were shown to be locally minimax
optimal in certain regions of the parameter space. There has also been
some recent focus on obtaining instance-dependent guarantees in online
reinforcement learning settings~\cite{maillard2014hard}. This has
resulted in more practically useful algorithms that provide, for
instance, horizon-independent regret bounds for certain episodic
MDPs~\cite{zanette2019tighter,jiang2018open}, thereby improving upon
worst-case bounds~\cite{AzaOsbMun17}.  Recent work has also
established some instance-dependent bounds, albeit not sharp over the
whole parameter space, for the problem of state-action value function
estimation in Markov decision processes, for both ordinary
$Q$-learning~\cite{wainwright2019stochastic} and a variance-reduced
improvement~\cite{wainwright2019variance}. \nocite{sidford2018near}


\subsection{Contributions}
\label{sec:contrib}

In this paper, we study stochastic approximation algorithms for
evaluating the value function of a Markov reward process in the
discounted setting.  Our goal is to provide a sharp characterization
of performance in the $\ell_\infty$-norm, for procedures that are
given access to state transitions and reward samples under the
generative model.  In practice, temporal difference learning is
typically applied with an additional layer of (linear) function
approximation.  In the current paper, so as to bring the instance dependence into sharp focus, we study the algorithms without this function approximation step.  In this context, we tell a story with three parts, as detailed below:

\paragraph{Local minimax lower bounds:}   Global minimax analysis
provides bounds that hold uniformly over large classes of models.  In
this paper, we seek to gain a more refined understanding of how the
difficulty of policy evaluation varies as a function of the instance.
In order to do so, we undertake an analysis of the local minimax risk
associated with a problem.  We first prove an asymptotic statement
(Proposition~\ref{prop:lower-bound-local-asymp-minimax-risk}) that
characterizes the local minimax risk up to a logarithmic factor; it
reveals the relevance of two functionals of the instance that we
define.  In proving this result, we make use of the classical
asymptotic minimax theorem~\cite{Hajek72, LeCam72, LeCamYa00}.  We
then refine this analysis by deriving a \emph{non-asymptotic} local
minimax bound, as stated in
Theorem~\ref{thm:lower-bound-local-minimax-risk-alt}, which is derived
using the non-asymptotic local minimax framework of Cai and
Low~\cite{cai2004adaptation}, an approach that builds upon the seminal concept of hardest local alternatives that can be traced back to Stein~\cite{Stein56a}.

\paragraph{Non-asymptotic suboptimality of iterate averaging:}
Our local minimax lower bounds raise a natural question: Do standard
procedures for policy evaluation achieve these instance-specific
bounds?  In Section~\ref{sec:suboptimality_of_averaging}, we address
this question for the TD(0) algorithm with iterate averaging.  Via a
careful simulation study, we show that for many popular stepsize
choices, the algorithm \emph{fails} to achieve the correct instance-dependent rate in the non-asymptotic setting, even when the sample size is quite large. This is true for both the constant stepsize, as well as polynomial stepsizes of various orders. Notably, the algorithm
with polynomial stepsizes of certain orders achieves the local risk in the asymptotic setting (see Theorem 1).

\paragraph{Non-asymptotic optimality of variance reduction:}

In order to remedy this issue with iterate averaging, we propose and
analyze a variant of TD learning with variance reduction, showing both
through theoretical (see Theorem 2) and numerical results (see
Figure~\ref{fig:VRPE_simulation}) that this algorithm achieves the
correct instance-dependent rate provided the sample size is larger
than an explicit threshold.  Thus, this algorithm is provably better
than TD$(0)$ with iterate averaging.

\subsection{Notation}

For a positive integer $n$, let $[n] := \{1, 2, \ldots, n\}$. For a
finite set $S$, we use $|S|$ to denote its cardinality.  We use $c, C,
c_1, c_2, \dots$ to denote universal constants that may change from
line to line.  We let $\ones$ denote the all-ones vector in
$\real^{\Dim}$.  Let $e_j$ denote the $j$th standard basis vector in
$\real^{\Dim}$.  We let $v_{(i)}$ denote the $i$-th order statistic of
a vector $v$, i.e., the $i$-th largest entry of~$v$.  For a pair of
vectors $(u, v)$ of compatible dimensions, we use the notation $u
\preceq v$ to indicate that the difference vector $v - u$ is
entrywise non-negative. The relation $u \succeq v$ is defined
analogously. We let $|u|$ denote the entrywise absolute value of a
vector $u \in \real^{\Dim}$; squares and square-roots of vectors are,
analogously, taken entrywise. Note that for a positive scalar
$\lambda$, the statements $|u | \preceq \lambda \cdot \ones$ and $\| u
\|_{\infty} \leq \lambda$ are equivalent. Finally, we let $\|
\mathbf{M} \|_{1, \infty}$ denote the maximum $\ell_1$-norm of the
rows of a matrix $\mathbf{M}$, and refer to it as the $(1,
\infty)$-operator norm of a matrix.  More generally, for scalars $q, p \geq 1$, 
we define $\matrixnorm{\mathbf{M}}_{p, q} \defeq \sup_{\norm{x}_p \le 1} \norm{ \mathbf{M}
  x}_q$. We let $\mathbf{M}^{\dagger}$ denote the Moore-Penrose
pseudoinverse of a matrix $\mathbf{M}$.


\section{Background and problem formulation}

We begin by introducing the basic mathematical formulation of Markov reward processes (MRPs) and generative observation models.

\subsection{Markov reward processes and value functions}

We study MRPs defined on a finite set of $\Dim$ states, which we index
by the set $[\Dim] \equiv \{1, 2, \ldots, \Dim\}$. The
state evolution over time is determined by a set of transition
functions, $\{P(\cdot | i), \; i \in [\Dim]\}$. Note that each such transition
function can be naturally associated with a $\Dim$-dimensional vector;
denote the $i$-th such vector as $p_i$. We let $\mathbf{P} \in[0,1]^{D
  \times D}$ denote a row-stochastic (Markov) transition matrix, where
row $i$ of this matrix contains the vector $p_{i}$. Also associated
with an MRP is a population reward function, $r: [\Dim] \mapsto
\mathbb{R}$, possessing the semantics that a transition from state $i$ results in the reward
$r(i)$. For convenience, we engage in a minor abuse of notation by
letting $r$ also denote a vector of length $\Dim$; here $r_{i}$
corresponds to the reward obtained at state $i$.

We formulate the long-term value of a state in the MRP in terms of the infinite-horizon, discounted reward. This value function (denoted
here by the vector $\thetastar \in \real^{\Dim}$) can be computed as
the unique solution of the Bellman fixed-point relation, $\thetastar =
r + \discount \mathbf{P} \thetastar$.

\subsection{Observation model}
\label{SecObs}

In the learning setting, the pair $(\mathbf{P}, r)$ is unknown, and we
accordingly assume access to a black box that generates samples from the
transition and reward functions. In this paper, we operate under a
setting known as the synchronous or generative
setting~\cite{KeaSin99}; this setting is also often referred to as the
``i.i.d.\ setting'' in the policy evaluation literature. For a given sample
index, $k \in \{1,2, \ldots, \nobs\}$ and for each state $j \in[D],$ we observe
a random next state
\begin{subequations}
  \label{eq:obs}
  \begin{align}
    \label{eq:obs-trans}
    X_{k, j} \sim P(\cdot | j) \qquad \mbox{for $j \in [\Dim]$.}
  \end{align}
We collect these transitions in a matrix $\Zsamp_k$, which by
definition contains one $1$ in each row: the $1$ in the $j$-th row
corresponds to the index of state $X_{k, j}$. We also observe a random
reward vector $R_{k} \in \real^{\Dim}$, where the rewards are
generated independently across states with\footnote{All of our upper
  bounds extend with minor modifications to the sub-Gaussian reward
  setting.}
\begin{align}
  \label{eq:obs-rew}
R_{k, j} \sim \mathcal{N}(r_{j}, \sigma_r^2).
\end{align}
\end{subequations}

Given these samples, define the $k$-th (noisy) linear operator
$\That_k: \real^{\Dim} \mapsto \real^{\Dim}$ whose evaluation at the
point $\theta$ is given by
\begin{align}
\label{eqn:That-defn}
\That_k (\theta) & = R_k + \discount \Zsamp_k \theta.
\end{align}
The construction of these operators is inspired by the fact that we
are interested in computing the fixed point of the population operator,
\begin{align}
\label{eqn:Tstar-defn}
\Tstar: \theta \mapsto r + \discount \Pmat \theta,
\end{align}
and a classical and natural way to do so is via a form of stochastic
approximation known as temporal difference learning, which we describe
next.

\subsection{Temporal difference learning and its variants}

Classical temporal difference (TD) learning algorithms are parametrized by a sequence of stepsizes, $\{\tdstep_k
\}_{k \geq 1}$, with $\tdstep_k \in (0, 1]$.  Starting with an
  initial vector $\theta_1 \in \real^\Dim$, the TD updates take the
  form
\begin{align}
  \label{eq:TD}
\theta_{k + 1} = (1 - \tdstep_k) \theta_k + \tdstep_k \That_k
(\theta_k) \quad \mbox{for $k = 1, 2, \ldots$.}
\end{align}
In the sequel, we discuss three popular stepsize choices:
\begin{subequations}
  \begin{align}
    \texttt{Constant stepsize:} & \qquad \tdstep_k = \tdstep, \qquad 
    \text{where} \qquad 0 < \tdstep \leq \tdstep_{\max}.
    \label{eqn:constant_step} \\
    \texttt{Polynomial stepsize:} & \qquad \tdstep_k = \frac{1}{k^\omega} \qquad 
    \mbox{for some $\omega \in (0, 1)$.}
    \label{eqn:poly_step} \\
    \texttt{Recentered-linear stepsize:} & \qquad
    \tdstep_k = \frac{1}{1 + (1 - \discount)k}.
    \label{eqn:recentered-linear} 
  \end{align}   
\end{subequations}

In addition to the TD sequence~\eqref{eq:TD}, it is also natural to
perform \emph{Polyak-Ruppert averaging}, which produces a parallel
sequence of averaged iterates
\begin{align}
\label{eq:TD-avg}
\thetatil_{k} = \frac{1}{k} \sum_{j = 1}^k \theta_j \quad \mbox{for $k
  = 1, 2, \ldots$.}
\end{align}
Such averaging schemes were introduced in the context of general
stochastic approximation by
Polyak~\cite{polyak1992acceleration} and Ruppert~\cite{Rup88}.  A large body of theoretical literature demonstrates that such an
averaging scheme improves the rates of convergence of stochastic
approximation when run with overly ``aggressive'' stepsizes.


\section{Main results}
\label{SecMain}

We turn to the statements of our main results and discussion of
their consequences.  All of our statements involve certain measures of
the local complexity of a given problem, which we introduce first.  We
then turn to the statement of lower bounds on the $\ell_\infty$-norm
error in policy evaluation.  In Section~\ref{SecLower}, we prove two
lower bounds. Our first result, stated as
Proposition~\ref{prop:lower-bound-local-asymp-minimax-risk}, is
asymptotic in nature (holding as the sample size $\Nsamp \rightarrow
+\infty$).  Our second lower bound, stated as
Theorem~\ref{thm:lower-bound-local-minimax-risk-alt}, provides a
result that holds for a range of finite sample sizes.  Given these
lower bounds, it is then natural to wonder about known algorithms that
achieve them.  Concretely, does the TD$(0)$ algorithm combined with
Polyak-Ruppert averaging achieve these instance-dependent bounds?  In
Section~\ref{sec:suboptimality_of_averaging}, we undertake a careful
empirical study of this question, and show that in the non-asymptotic
setting, this algorithm fails to match the instance-dependent bounds.
This finding sets up the analysis in Section~\ref{sec:VRPE}, where we
introduce a variance-reduced version of TD$(0)$, and prove that it
does achieve the instance-dependent lower bounds from
Theorem~\ref{thm:lower-bound-local-minimax-risk-alt} up to a
logarithmic factor in dimension.


\paragraph{Local complexity measures:}
Recall the generative observation model described in
Section~\ref{SecObs}.  For a transition matrix $\Pmat$, we write
$\Zsamp \sim \Pmat$ to denote a draw of a random matrix with $\{0,1 \}$ entries
and a single one in each row (with the position of the one in row
$\Zsamp_j$ determined by sampling from the multinomial distribution
specified by $p_j$). For a fixed vector $\theta \in
\real^{\Dim}$, note that $(\Zsamp - \Pmat) \theta$ is a random vector
in $\real^\Dim$, and define its covariance matrix as follows:
\begin{align}
\label{eqn:def-Sigma}
\Sigma_{\Pmat}(\theta) = \cov_{\Zsamp \sim \Pmat} \left((\Zsamp-\Pmat)
\theta\right).
\end{align}
We often use $\Sigma(\theta)$ as a shorthand for $\Sigma_{\Pmat}
(\theta)$ when the underlying transition matrix $\Pmat$ is clear from
the context.

With these definitions in hand, define the complexity measures
 \begin{subequations}
 \label{eq:local-complexity}
 \begin{align}
\nu(\Pmat, \theta) & \mydefn \max_{\ell \in [D]}\left(e_{\ell}^{\top}
(\IdMat - \discount \Pmat)^{-1} \Sigma (\theta)(\IdMat - \discount
\Pmat)^{-\top} e_{\ell}\right)^{1/2}, \quad \mbox{and} \\
\rho(\Pmat, r) & \mydefn \sigma_r \norm{(\IdMat - \discount
  \Pmat)^{-1}}_{2, \infty} \; \equiv \; \sigma_r \max_{\|u\|_2 = 1}
\|(\IdMat - \discount \Pmat)^{-1} u\|_\infty.
 \end{align}

 Note that $\nu(\Pmat, \theta)$ corresponds to the maximal variance of
 the random vector $(\IdMat - \discount \Pmat)^{-1} (\Zsamp - \Pmat)
 \theta$.  As we demonstrate shortly, the quantities $\nu(\Pmat,
 \thetastar)$ and $\rho(\Pmat, r)$ govern the local complexity of
 estimating the value function $\thetastar$ induced by the instance
 $(\Pmat, r)$ under the observation model~\eqref{eq:obs}.  A portion
 of our results also involve the quantity
\begin{align}
  b(\theta) & \mydefn \frac{\| \theta \|_{\myspan}}{1 - \discount},
\end{align}
\end{subequations}
 where $\|\theta\|_{\myspan} = \max \limits_{j \in [D]} \theta_j -
 \min \limits_{j \in [D]} \theta_j$ is the span seminorm.


\subsection{Local minimax lower bound}
\label{SecLower}

Throughout this section, we use the letter $\problem$ to denote an
individual problem instance, $\problem = (\Pmat, r)$, and use $\theta(\problem)
:= \thetastar = (\IdMat - \discount \Pmat)^{-1}r$ to denote the
\emph{target} of interest. The aim of this section is to establish
\emph{instance-specific} lower bounds for estimating
$\theta(\problem)$ under the observation model~\eqref{eq:obs}. In
order to do so, we adopt a local minimax approach.

The remainder of this the section is organized as follows.  In
Section~\ref{sec:asymptotic-lower-bound}, we prove an asymptotic local
minimax lower bound, valid as the sample size $\Nsamp$ tends to
infinity. It gives an explicit Gaussian limit for the rescaled error
that can be achieved by any procedure.  The asymptotic covariance in
this limit law depends on the problem instance, and is very closely
related to the functionals $\nu(\Pmat,\thetastar)$ and $\rho(\Pmat,
r)$ that we have defined.  Moreover, we show that this limit can be
achieved---in the asymptotic limit---by the TD algorithm combined with
Polyak-Ruppert averaging.  While this provides a useful sanity check, in practice we implement estimators using a finite number of samples $\Nsamp$, so it is important to obtain non-asymptotic lower bounds for a full understanding.  With this motivation,
Section~\ref{sec:non-asymptotic-lower-bound} provides a new,
\emph{non-asymptotic} instance-specific lower bound for the policy
evaluation problem.  We show that the quantities
$\nu(\Pmat,\thetastar)$ and $\rho(\Pmat, r)$ also cover the
instance-specific complexity in the finite-sample setting.  In proving
this non-asymptotic lower bound, we build upon techniques in the
statistical literature based on constructing hardest one-dimensional
alternatives~\cite{Stein56a, Birge83, DonohoLi87, DonohoLi91a,
  CaiLo15}.  As we shall see in later sections, while the TD algorithm
with averaging is instance-specific optimal in the asymptotic setting,
it \emph{fails} to achieve our non-asymptotic lower bound. 

\subsubsection{Asymptotic local minimax lower bound}
\label{sec:asymptotic-lower-bound}

Our first approach towards an instance-specific lower bound is an
asymptotic one, based on classical local asymptotic minimax theory.
For regular and parametric families, the H\'{a}jek--Le Cam local
asymptotic minimax theorem~\cite{Hajek72, LeCam72, LeCamYa00} shows
that the Fisher information---an instance-specific
functional---characterizes a fundamental asymptotic limit.  Our model
class is both parametric and regular (cf.\ Eq.~\eqref{eq:obs}),
and so this classical theory applies to yield an asymptotic local
minimax bound.  Some additional work is needed to relate this statement to the more transparent complexity measures $\nu(\Pmat,\thetastar)$ and
$\rho(\Pmat, \reward)$ that we have defined.

In order to state our result, we require some additional notation. Fix
an instance $\problem = (\Pmat, \reward)$.  For any $\epsilon > 0$, we
define an $\epsilon$-neighborhood of problem instances by
\begin{align*}
\neighbor(\problem; \epsilon) = \left\{\problem' = (\Pmat', \reward'):
\norm{\Pmat - \Pmat'}_F + \norm{r - r'}_2 \le \epsilon \right\}.
\end{align*}
Adopting the $\ell_\infty$-norm as the loss function, the \emph{local
  asymptotic minimax risk} is given by
\begin{align}
\minimax_{\infty}(\problem) \equiv \minimax_{\infty}(\problem;
\norm{\cdot}_\infty) = \lim_{c \to \infty} \lim_{\nobs \to \infty}
\inf_{\hat{\theta}_\nobs} \sup_{\qproblem \in \neighbor(\problem;
  c/\sqrt{\nobs})} \E_{\qproblem} \left[ \sqrt{\nobs}
  \norm{\hat{\theta}_{\nobs} - \theta(\qproblem)}_\infty\right].
\end{align}
Here the infimum is taken over all estimators $\thetahat_{\nobs}$ that
are measurable functions of $\nobs$ i.i.d. observations drawn
according to the observation model~\eqref{eq:obs}.

Our first main result characterizes the local asymptotic risk
$\minimax_{\infty}(\problem) $ exactly, and shows that it is attained
by stochastic approximation with averaging. Recall the Polyak-Ruppert
(PR) sequence $\{ \thetatil_k \}_{k \geq 1}$ defined in
Eq.~\eqref{eq:TD-avg}, and let $\{ \thetatil^{\,\omega}_k \}_{k
  \geq 1}$ denote this sequence when the underlying SA algorithm is
the TD update with the polynomial stepsize
sequence~\eqref{eqn:poly_step} with exponent $\omega$.

\begin{proposition}
\label{prop:lower-bound-local-asymp-minimax-risk}
Let 
$Z \in \real^D$ be a multivariate Gaussian with zero mean and
covariance matrix
\begin{subequations}
\begin{align}
  (\IdMat - \discount \Pmat)^{-1} (\discount^2
  \Sigma_{\Pmat}(\theta(\problem) ) + \sigma_r^2 \IdMat) (\IdMat -
  \discount \Pmat)^{-\top}.
\end{align}
Then the local asymptotic minimax risk at problem instance $\problem$
is given by
\begin{align}
\label{EqnAsympLower} 
\minimax_{\infty}(\problem) = \E[ \| Z \|_\infty ].
\end{align}
Furthermore, for each problem instance $\problem$ and scalar $\omega
\in (1/2, 1)$, this limit is achieved by the TD algorithm with
an $\omega$-polynomial stepsize and PR-averaging:
\begin{align}
  \label{eq:SA-asymp-opt}
\lim_{\nobs \to \infty} \; \sqrt{\nobs} \cdot \E \left[ \|
  \thetatil^{\,\omega}_{\nobs} - \theta(\problem) \|_\infty \right] =
\E[ \| Z \|_\infty ].
\end{align}
\end{subequations}
\end{proposition}
With the convention that $\thetastar \equiv \theta(\problem)$, a short
calculation bounding the maximum absolute value of sub-Gaussian random
variables (see, e.g., Ex. 2.11 in Wainwright~\cite{wainwright2019high})
yields the sandwich relation
\begin{align*}
\discount \nu(\Pmat,\thetastar) + \rho(\Pmat, r) \leq \E [ \| Z
  \|_\infty ] \leq \sqrt{2 \log \Dim } \cdot \left( \discount
\nu(\Pmat,\thetastar) + \rho(\Pmat, r) \right),
\end{align*}
so that Proposition~\ref{prop:lower-bound-local-asymp-minimax-risk}
shows that, up to a logarithmic factor in dimension $\Dim$, the local
asymptotic minimax risk is entirely characterized by the functional
$\discount \nu(\Pmat,\thetastar) + \rho(\Pmat, r)$.

It should be noted that lower bounds similar to
Eq.~\eqref{EqnAsympLower} have been shown for specific classes of
stochastic approximation
algorithms~\cite{tsypkin1974attainable}. However, to the best of our
knowledge, a local minimax lower bound---one applying to any procedure
that is a measurable function of the observations---is not available in the existing literature. 

Furthermore, Eq.~\eqref{eq:SA-asymp-opt} shows that stochastic
approximation with polynomial stepsizes and averaging attains the
exact local asymptotic risk. Our proof of this result essentially
mirrors that of Polyak and Juditsky~\cite{polyak1992acceleration}, and
amounts to verifying their assumptions under the policy evaluation
setting.  Given this result, it is natural to ask if averaging is
optimal also in the non-asymptotic setting; answering this question is the focus of the next two sections of the paper.

\subsubsection{Non-asymptotic local minimax lower bound}
\label{sec:non-asymptotic-lower-bound}

Proposition~\ref{prop:lower-bound-local-asymp-minimax-risk} provides
an instance-specific lower bound on $\theta(\problem)$ that holds
asymptotically.  In order to obtain a non-asymptotic guarantee, we
borrow ideas from the non-asymptotic framework introduced by Cai and
Low~\cite{CaiLo15} for nonparametric shape-constrained inference.
Adapting their definition of local minimax risk to our problem
setting, given the loss function $L(\theta - \thetastar) = \|\theta -
\thetastar\|_\infty$, the (normalized) \emph{local non-asymptotic
  minimax risk} for $\theta(\cdot)$ at instance $\problem = (\Pmat,
\reward)$ is given by
\begin{align}
  \label{eq:minimax-risk}
\minimax_{\nobs}( \problem) = \sup_{\problem^\prime}
\inf_{\hat{\theta}_\nobs} \max_{\qproblem \in \{\problem,
  \problem^\prime\}} \sqrt{\nobs} \cdot \E_{\qproblem} \left[\|\thetahat_{\nobs}-
  \theta(\qproblem)\|_\infty \right].
\end{align}
Here the infimum is taken over all estimators $\thetahat_{\nobs}$ that
are measurable functions of $\nobs$ i.i.d. observations drawn
according to the observation model~\eqref{eq:obs}, and the
normalization by $\sqrt{\nobs}$ is for convenience.  The
definition~\eqref{eq:minimax-risk} is motivated by the notion of the
hardest one-dimensional alternative~\cite[Ch. 25]{VanDerVaart98}. Indeed, given an instance
$\problem$, the local non-asymptotic risk $\minimax_{\nobs}(\problem)$
first looks for the hardest alternative $\problem^\prime$ against
$\problem$ (which should be local around $\problem$), then measures
the worst-case risk over $\problem$ and its (local) hardest
alternative $\problem^\prime$.

With this definition in hand, we lower bound the local
non-asymptotic minimax risk using the complexity measures
$\nu(\Pmat,\thetastar)$ and $\rho(\Pmat, r)$ defined in
Eq.~\eqref{eq:local-complexity}:
\begin{theorem}
\label{thm:lower-bound-local-minimax-risk-alt}
There exists a universal constant $c > 0$ such that for any instance
$\problem = (\Pmat, \reward)$, the local non-asymptotic minimax risk
is lower bounded as
\begin{align}
\label{eqn:main-lower-bound}
\minimax_\nobs(\problem) \ge c \Big(
\discount \nu(\Pmat, \thetastar) + \rho(\Pmat, r) \Big).
\end{align}
This bound is valid for all sample sizes $\nobs$ that satisfy
\begin{align}
\label{eqn:N-large-condition}
\nobs \geq \nobs_0 \mydefn \max \left
\{\frac{\discount^2}{(1-\discount)^2},
\frac{b^2(\thetastar)}{\nu^2(\Pmat, \thetastar)} \right\}.
\end{align}
\end{theorem}

A few comments are in order.  First, it is natural to wonder about the
necessity of condition~\eqref{eqn:N-large-condition} on the sample
size $\Nsamp$ in our lower bound.  Our past work provides upper bounds
on the $\ell_\infty$-error of the plugin
estimator~\cite{pananjady2019value}, and these results also require a
bound of this type.  In fact, when the rewards are observed with noise
(i.e., for any $\sigma_r > 0$), the condition $\nobs \gtrsim
\frac{\discount^2}{(1 - \discount)^2}$ is natural, since it is
necessary in order to obtain an estimate of the value function with
$\order(1)$ error.  On the other hand, in the special case of
deterministic rewards ($\sigma_r = 0$), it is interesting to ask how
the fundamental limits of the problem behave in the absence of this
condition.

Second, note that the non-asymptotic lower
bound~\eqref{eqn:main-lower-bound} is closely connected to the
asymptotic local minimax bound from
Proposition~\ref{prop:lower-bound-local-asymp-minimax-risk}.  In
particular, for any sample size $\Nsamp$ satisfying the lower
bound~\eqref{eqn:N-large-condition}, our non-asymptotic lower
bound~\eqref{eqn:main-lower-bound} coincides with the asymptotic lower
bound~\eqref{EqnAsympLower} up to a constant factor.  Thus, it cannot
be substantially sharpened.  The finite-sample nature of the lower
bound~\eqref{eqn:main-lower-bound} is a powerful tool for assessing
optimality of procedures: it provides a performance benchmark that
holds over a large range of finite sample sizes $\Nsamp$.  Indeed, in
the next section, we study the performance of the TD
learning algorithm with Polyak-Ruppert averaging.  While this
procedure achieves the local minimax lower bound asymptotically, as
guaranteed by Eq.~\eqref{eq:SA-asymp-opt} in
Proposition~\ref{prop:lower-bound-local-asymp-minimax-risk}, it falls
short of doing so in natural \emph{finite-sample} scenarios.


\subsection{Suboptimality of averaging}
\label{sec:suboptimality_of_averaging}

Polyak and Juditsky~\cite{polyak1992acceleration} provide a general
set of conditions under which a given stochastic-approximation (SA)
algorithm, when combined with Polyak-Ruppert averaging, is guaranteed
to have asymptotically optimal behavior.  For the current problem, the
bound~\eqref{eq:SA-asymp-opt} in
Proposition~\ref{prop:lower-bound-local-asymp-minimax-risk}, which is
proved using the Polyak-Juditsky framework, shows that SA with
polynomial stepsizes and averaging have this favorable asymptotic
property.

However, asymptotic theory of this type gives no guarantees in the
finite-sample setting. In particular, suppose that we are given a
sample size $\Nsamp$ that scales as $(1-\discount)^{-2}$, as specified
in our lower bounds.  Does the averaged TD$(0)$ algorithm exhibit
optimal behavior in this non-asymptotic setting?  In this section, we
answer this question in the negative.  More precisely, we describe a
parameterized family of Markov reward processes, and provide careful
simulations that reveal the suboptimality of TD without
averaging.

\begin{figure}[ht]
  \begin{center}
    \begin{tabular}{ccc}
          \raisebox{0.2in}{\widgraph{0.35\textwidth}{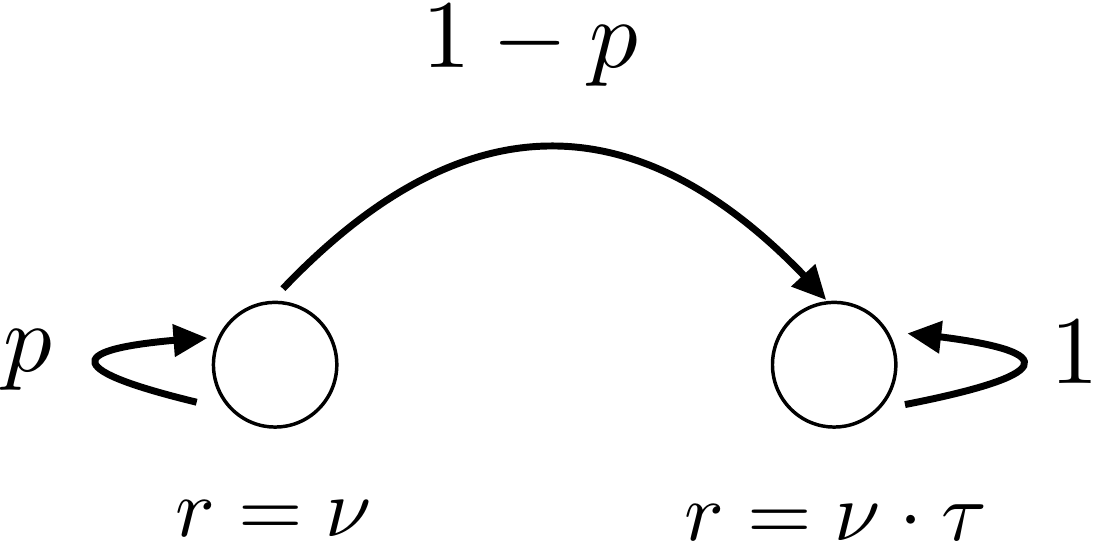}}
    \end{tabular}
    \caption{Illustration of the 2-state MRP used in the simulation.
      The triple of scalars $(p, \nu, \tau)$, along with the discount
      factor $\discount$, are parameters of the construction.  The chain
      remains in state $1$ with with probability $p$ and transitions to
      state $2$ with probability $1-p$; on the other hand, state $2$
      is absorbing.  The rewards in states $1$ and $2$ are
      deterministic, specified by $\nu$ and $\nu \tau$, respectively.}
    \label{fig:mrp}
  \end{center}
\end{figure}

\subsubsection{A simple construction}

The lower bound in
Theorem~\ref{thm:lower-bound-local-minimax-risk-alt} predicts a range
of behaviors depending on the pair $\nu(\Pmat, \thetastar)$ and
$\rho(\Pmat, r)$.  In order to observe a large subset of these behaviors, it suffices to consider a very simple
MRP, $\problem = (\Pmat, r)$ with $\Dim = 2$ states, as illustrated in
Figure~\ref{fig:mrp}. In this MRP, the transition matrix $\Pmat \in
\real^{2 \times 2}$ and reward vector $r \in \real^2$ take the form
\begin{align*}
\Pmat & = \begin{bmatrix} p & 1 - p \\ 0 & 1 
\end{bmatrix}, \quad \mbox{and} \quad \reward = \begin{bmatrix} \nu \\ \nu \tau 
\end{bmatrix}.
\end{align*}
Here the triple $(p, \nu, \tau)$, along with the discount factor
$\discount \in [0, 1)$, are parameters of the construction.

In order to parameterize this MRP in a scalarized manner, we vary the
triple $(p, \nu, \tau)$ in the following way.  First, we fix a scalar
$\parMRP \geq 0$, and then we set
 \begin{align*}
 p = \tfrac{4 \discount - 1}{3 \discount}, \qquad \lbmaxrew = 1 \quad
 \text{ and } \quad \; \tau = 1 - (1 - \discount)^{\parMRP}.
 \end{align*}
Note that this sub-family of MRPs is fully parametrized by the pair
$(\discount, \parMRP)$.  Let us clarify why this particular
scalarization is interesting. It can be shown via simple calculations
that the underlying MRP satisfies
\begin{align*}
 \nu(\Pmat, \thetastar) \sim \left( \frac{1}{1 - \discount}
 \right)^{1.5 - \parMRP}, \quad \rho(\Pmat, r) = 0 \quad \text{and}
 \quad \;\; b(\thetastar) \sim \left( \frac{1}{1 - \discount}
 \right)^{2 - \parMRP},
\end{align*}
where $\sim$ denotes equality that holds up to a constant pre-factor.
Consequently, by Theorem~\ref{thm:lower-bound-local-minimax-risk-alt}
the minimax risk, measured in terms of the $\ell_{\infty}$-norm,
satisfies
\begin{align}
\label{eqn:optimal_rate_in_alpha}
\minimax_\nobs(\problem) \geq c \cdot \left(
\frac{1}{1 - \discount} \right)^{1.5 - \parMRP}.
\end{align}
Thus, it is natural to study whether the TD$(0)$ algorithm with PR
averaging achieves this error.


\subsubsection{A simulation study}

In order to compare the behavior of averaged TD with the lower
bound~\eqref{eqn:optimal_rate_in_alpha}, we performed a series of
experiments of the following type.  For a fixed parameter $\parMRP$ in
the range $[0, 1.5]$, we generated a range of MRPs with different
values of the discount factor $\discount$. For each value of the
discount parameter $\discount$, we consider the problem of estimating
$\thetastar$ using a sample size $\Nsamp$ set to be one of two
possible values: namely, $\nobs \in \left\{ \lceil \frac{8}{(1 -
      \discount)^2} \rceil, \lceil \frac{8}{(1 -
      \discount)^3} \rceil \right\}$.

\begin{figure}[ht!]
  \begin{center}
    \begin{tabular}{ccc}
      \widgraph{0.4\textwidth}{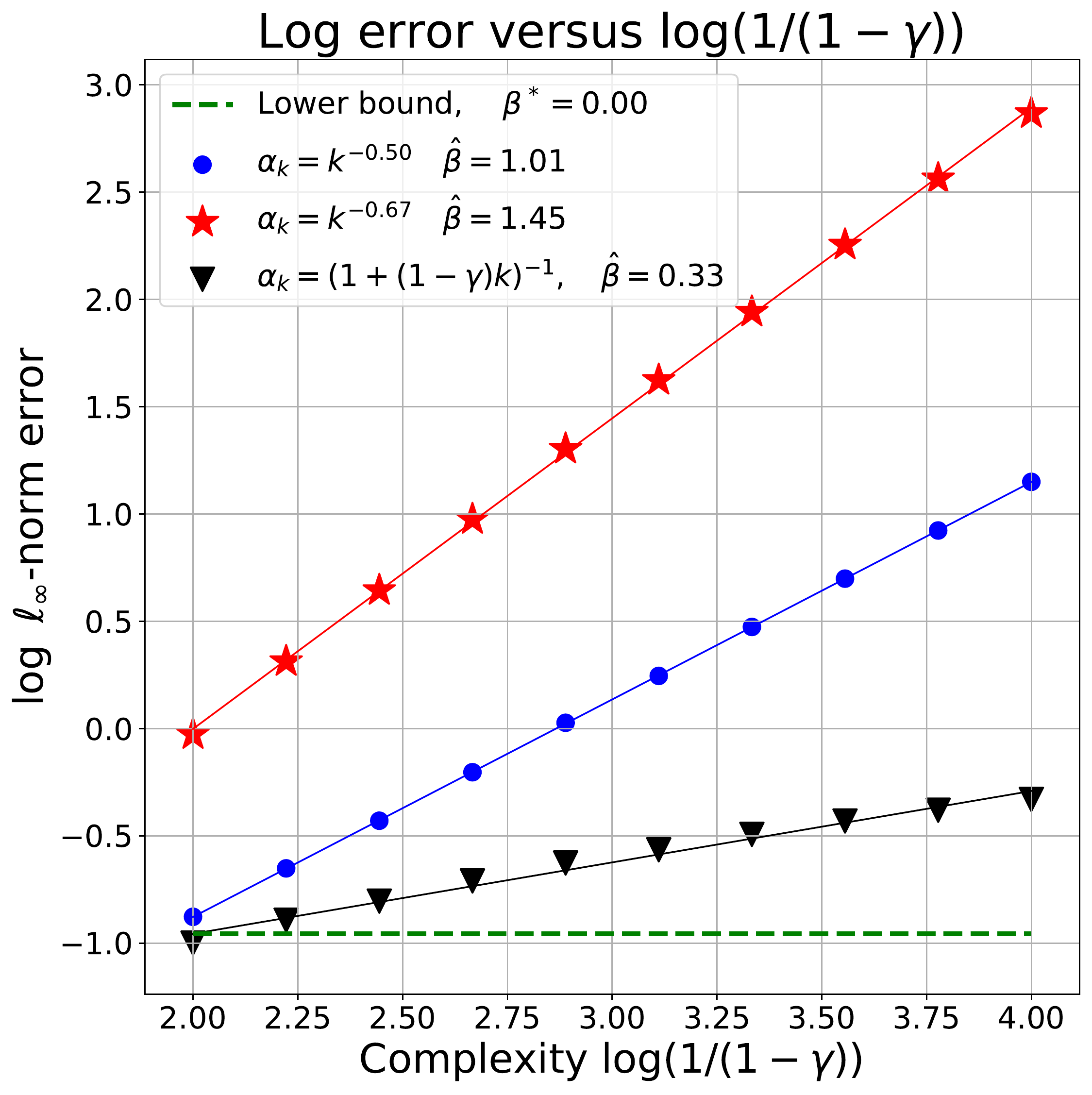}
      &\hspace{.1in} &
      \widgraph{0.4\textwidth}{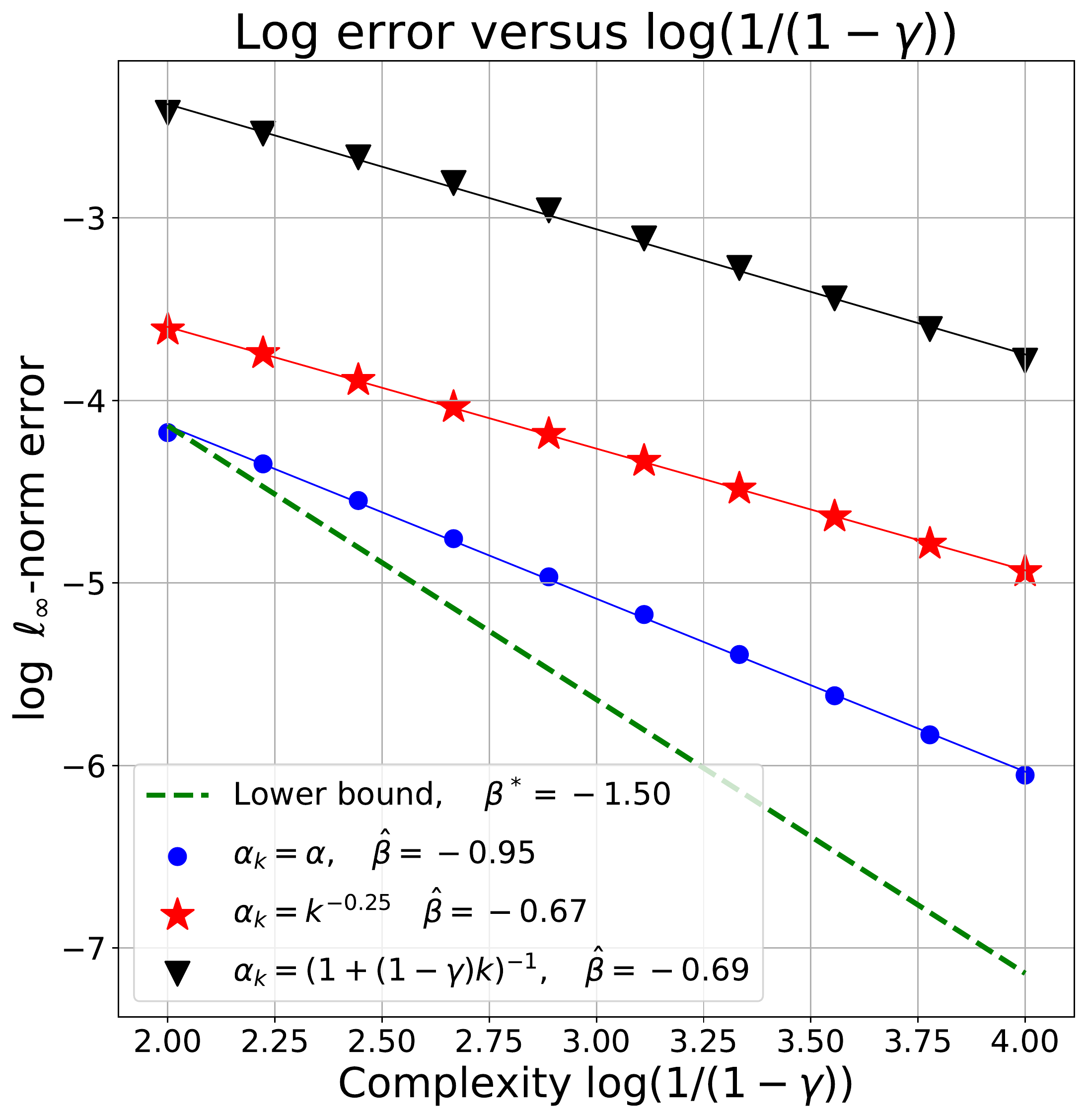}
      \\
    (a) $\parMRP = 0.5, \;\; \Nsamp = \lceil \frac{8}{(1 -
      \discount)^2} \rceil$. && (b)
      $\parMRP = 1.5, \;\; \Nsamp = \lceil \frac{8}{(1 -
      \discount)^3} \rceil$.
    \end{tabular}
  \caption{Log-log plots of the $\ell_\infty$-error versus the
    discount complexity parameter $1 / (1 - \discount)$ for various
    algorithms.  Each point represents an average over $1000$ trials,
    with each trial simulations are for the 2-state MRP depicted in
    Figure~\ref{fig:mrp} with the parameter choices $p =
    \tfrac{4\discount - 1}{3\discount}$, $\lbmaxrew = 1$ and $\tau = 1
    - (1 - \discount)^{\parMRP}$.  We have also plotted the
    least-squares fits through these points, and the slopes of these
    lines are provided in the legend. In particular, the legend
    contains the stepsize choice for averaged SA (denoted as
    $\tdstep_k$), the slope $\hat{\beta}$ of the least-squares line,
    and the ideal value $\betastar$ of the slope computed in
    equation~\ref{eqn:ideal-slope}. We also include the lower bound
    predicted by Theorem~\ref{thm:lower-bound-local-minimax-risk-alt}
    for these examples as a dotted line for comparison purposes.
    Logarithms are to the natural base.}
  \label{fig:avg_simulation}
  \end{center}
\end{figure}

In Figure~\ref{fig:avg_simulation}, we plot the $\ell_\infty$-error of
the averaged SA, for constant stepsize~\eqref{eqn:constant_step},
polynomial-decay stepsize~\eqref{eqn:poly_step} and recentered linear
stepsize~\eqref{eqn:recentered-linear}, as a function of
$\discount$. The plots show the behavior for \mbox{$\parMRP \in \{0.5,
  1.5\}$.}  Each point on each curve is obtained by averaging $1000$
Monte Carlo trials of the experiment.  Note that from our lower bound
calculations above~\eqref{eqn:optimal_rate_in_alpha}, the log
$\ell_\infty$-error is related to the complexity $\log \big(
\tfrac{1}{1 - \discount} \big)$ in a linear fashion; we use $\beta^*$
to denote the slope of this idealized line.  Simple algebra yields
\begin{align}
  \betastar = \frac{1}{2} - \parMRP \quad \text{for} \quad \nobs =
  \frac{1}{(1 - \discount)^2}, \quad \text{and} \quad \betastar =
  -\parMRP \qquad \text{for} \quad \nobs = \frac{1}{(1 - \discount)^3}.
  \label{eqn:ideal-slope}
\end{align}
In other words, for an algorithm which achieves the lower bound
predicted by our theory, we expect a linear relationship between the
log $\ell_\infty$-error and log discount complexity $\log \big(
\tfrac{1}{1 - \discount} \big)$, with the slope~$\betastar$.

Accordingly, for the averaged SA estimators with the stepsize
choices in~\eqref{eqn:constant_step}-\eqref{eqn:recentered-linear}, we
performed a linear regression to estimate the slopes between the log
$\ell_\infty$-error and the log discount-complexity $\log \big(
\tfrac{1}{1 - \discount} \big)$. The plot legend reports the stepsize
choices $\tdstep_{k}$ and the slope $\betahat$ of the fitted
regression line.  We also include the lower bound in the plots, as a
dotted line along with its slope, for a visual comparison. We see that
the slopes corresponding to the averaged SA algorithm are higher
compared to the ideal slopes of the dotted lines. Stated differently,
this means that the averaged SA algorithm does not achieve the lower
bound with either the constant step or the polynomial-decay step.
Overall, the simulations provided in this section demonstrate that 
the averaged SA algorithm, although guaranteed to be
asymptotically optimal by Eq.~\eqref{eq:SA-asymp-opt} in
Proposition~\ref{prop:lower-bound-local-asymp-minimax-risk},
does not yield the ideal non-asymptotic behavior.


\subsection{Variance-reduced policy evaluation}
\label{sec:VRPE}

In this section, we propose and analyze a variance-reduced version of
the TD learning algorithm. As in standard
variance-reduction schemes, such as
SVRG~\cite{johnson2013accelerating}, our algorithm proceeds in
epochs. In each epoch, we run a standard stochastic approximation
scheme, but we recenter our updates in order to reduce their
variance. The recentering uses an empirical approximation to the
population Bellman operator $\Tstar$.

We describe the behavior of the algorithm over epochs by a
sequence of operators, $\left \{ \V_{\epiter} \right \}_{\epiter \geq
1}$, which we define as follows. At epoch $\epiter$, the method uses a vector $\thetabar_\epiter$ in order to recenter the update, where
the vector $\thetabar_\epiter$ should be understood as the best current
approximation to the unknown vector $\thetastar$.
In the ideal scenario, such a recentering would involve the quantity
$\Tstar(\thetabar_\epiter)$, where $\Tstar$ denotes the population
operator previously defined in Eq.~\eqref{eqn:Tstar-defn}.  Since
we lack direct access to the population operator $\Tstar$, however, we use the Monte Carlo approximation
\begin{align}
\label{eqn:MC_avg_operator}
\Ttil_{\Nm}(\thetabar_\epiter) & \mydefn \frac{1}{\Nm} \sum_{i \in \Dm
} \That_{i}(\thetabar_\epiter),
\end{align}
where the empirical operator $\That_i$ is defined in
Eq.~\eqref{eqn:That-defn}.  Here the set $\Dm$ is a collection of
$\Nm$ i.i.d.\ samples, independent of all other randomness.

Given the pair $(\thetabar_\epiter, \Ttil_{\Nm}(\thetabar_\epiter))$
and a stepsize $\tdstep \in (0,1)$, we define the operator
$\V_{\iter}$ on $\real^{\usedim}$ as follows:
\begin{align}
\label{eqn:VR-SA-update}
\theta \mapsto \V_{k}\left(\theta ; \tdstep, \thetabar_\epiter, \Ttil_{\Nm} \right):=(1-\tdstep) \theta +
\tdstep\left\{\That_{k}(\theta) - \That_{k}(\thetabar_\epiter) + \Ttil_{\Nm} (\thetabar_\epiter)\right\}.
\end{align}
As defined in Eq.~\eqref{eqn:That-defn}, the quantity
$\That_\iter$ is a stochastic operator, where the randomness is independent of the set
of samples $\Dm$ used to define $\Ttil_{\Nm}$. Consequently, the
stochastic operator $\That_\iter$ is independent of the recentering
vector $\Ttil_{\Nm}(\thetabar_\epiter)$. Moreover, by construction,
for each $\theta \in \real^{\Dim}$, we have
\begin{align*}
\Exs \left[ \That_k (\theta) - \That_k (\thetabar_\epiter) + \Ttil_{\Nm} (\thetabar_\epiter)  \right]
= \Tstar(\theta).
\end{align*}
Thus, we see that $\V_k$ can be seen as an unbiased stochastic
approximation of the population-level Bellman operator.  As will be
clarified in the analysis, the key effect of the recentering steps is
to reduce its associated variance.


\subsubsection{A single epoch}
Based on the variance-reduced policy evaluation update defined in
Eq.~\eqref{eqn:VR-SA-update}, we are now ready to define a single
epoch of the overall algorithm.  We index epochs using the integers
\mbox{$\epiter = 1, 2, \ldots, \NumEpoch$,} where $\NumEpoch$
corresponds to the total number of epochs to be run. Epoch $m$
requires as inputs the following quantities:
\begin{itemize}
\item a vector $\thetabar$, which is chosen to be the output
  of the previous epoch,
\item a positive integer $\epochLen$ denoting the number of steps
  within the given epoch,
\item a positive integer $\Nm$ denoting the number of samples used to
  calculate the Monte Carlo update~\eqref{eqn:MC_avg_operator},
\item a sequence of stepsizes $\left\{ \tdstep_k \right\}_{k \geq 1}^{\Klambda}$
  with $\tdstep_k \in (0, 1)$, and
\item a set of fresh samples $\EpochSamples$, with $|\Em| = \Nm +
  \Klambda$. The first $\Nm$ samples are used to define the dataset
  $\Dm$ that underlies the Monte Carlo
  update~\eqref{eqn:MC_avg_operator}, whereas the remaining $\Klambda$
  samples are used in the $\Klambda$ steps within each epoch.
\end{itemize}
We summarize the operations within a single epoch in
{Algorithm~\ref{Algo:one_epoch}}.

\begin{algorithm}[ht!]
\caption{ $\;\;\;\;\;$ RunEpoch $(\thetabar ; \epochLen, \Nm, 
\left\{ \tdstep_{k} \right\}_{k = 1}^{\Klambda}, \EpochSamples$)}\label{Algo:one_epoch}
\begin{algorithmic}[1]
\STATE{Given (a) Epoch length $\epochLen \quad$, (b) Recentering vector $\thetabar \quad$, (c) Recentering sample size $\Nm$, \\ 
(d) Stepsize sequence $\left\{ \tdstep_{k} \right\}_{k \geq 1}^{\Klambda}$,
(e) Samples \EpochSamples}
\vspace{10pt}
\STATE{Compute the recentering quantity
  $\Ttil_{\Nm}(\thetabar)  \mydefn \frac{1}{\Nm} \sum \limits_{i \in \Dm
} \That_{i}(\thetabar)$} 
\STATE{ Initialize $\theta_{1}=\thetabar$ }
\FOR{ $k = 1, 2, \ldots, \epochLen$ } \STATE{Compute the
  variance-reduced update:
\begin{align*}
\theta_{k+1} = \V_{k}\left(\theta_{k} ; \tdstep_{k}, \thetabar,
\Ttil_{\Nm}\right)
\end{align*}
}
\ENDFOR
\end{algorithmic}
\end{algorithm}

The choice of the stepsize sequence $\{\tdstep_k\}_{k \geq 1}$ is
crucial, and it also determines the epoch length $\epochLen$. Roughly
speaking, it is sufficient to choose a large enough epoch length to
ensure that the error is reduced by a constant factor in each
epoch. In Section~\ref{sec:VRPE-thm-section} to follow, we study three
popular stepsize choices---the constant
stepsize~\eqref{eqn:constant_step}, the polynomial
stepsize~\eqref{eqn:poly_step} and the recentered linear
stepsize~\eqref{eqn:recentered-linear}---and provide lower bounds on
the requisite epoch length in each case.


\subsubsection{Overall algorithm}

We are now ready to specify our variance-reduced policy-evaluation (VRPE) algorithm.  The overall
algorithm has five inputs: (a) an integer $\TotalEpochs$, denoting the
number of epochs to be run, (b) an integer $\Klambda$, denoting the
length of each epoch, (c) a sequence of sample sizes $\{ \Nm \}_{m =
  1}^{\TotalEpochs}$ denoting the number of samples used for
recentering, (d) Sample batches $\{ \EpochSamples \} _{m =
  1}^{\TotalEpochs}$ to be used in $\epiter$ epochs, and (e) a
sequence of stepsize $\{ \tdstep_k \}_{k \geq 1}$ to be used in each
epoch.  Given these five inputs, we summarize the overall procedure in
Algorithm~\ref{Algo:variance_reduced_SA}:

\begin{algorithm}[ht]
\caption{ $\;\;\;\;\;$ Variance-reduced policy evaluation (VRPE)
}\label{Algo:variance_reduced_SA}
\begin{algorithmic}[1]
\STATE{Given (a) Number of epochs $\NumEpoch$, (b) Epoch length
  $\epochLen$ $\quad$, (c) Recentering sample sizes
  $\left\{\Nm\right\}_{\epiter=1}^\NumEpoch$, (d) Sample batches
  \EpochSamples, for $m = 1, \ldots, \TotalEpochs$,  (e) Stepsize
  $\{\tdstep_{k}\}_{k = 1}^{\Klambda}$}
  \vspace{10pt}
  \STATE{Initialize at $\thetabar_1$}
\FOR{ $\epiter = 1, 2, \ldots, \NumEpoch$ }
\STATE{
  $\thetabar_{\epiter + 1} =
  \operatorname{RunEpoch}\left(\thetabar_{\epiter} ; \epochLen,
  \Nm, \{ \tdstep \}_{k = 1}^{\Klambda}, \EpochSamples \right)$ 
  } 
  \ENDFOR \STATE{ Return $\thetabar_{\NumEpoch + 1}$ as
  the final estimate}
\end{algorithmic}
\end{algorithm}
In the next section, we provide a detailed description on how to
choose these input parameters for three popular stepsize
choices~\eqref{eqn:constant_step}--\eqref{eqn:recentered-linear}.
Finally, we reiterate that at epoch $\epiter$, the algorithm uses $\Nm
+ \Klambda$ new samples, and the samples used in the epochs are
independent of each other. Accordingly, the total number of samples
used in $\TotalEpochs$ epochs is given by $\Klambda \TotalEpochs +
\sum_{\epiter = 1}^{\TotalEpochs} \Nm$.

 
\subsubsection{Instance-dependent guarantees}
\label{sec:VRPE-thm-section}

Given a desired failure probability, $\delta \in (0,1)$, and a total
sample size $\nobs$, we specify the following choices of parameters in
Algorithm~\ref{Algo:variance_reduced_SA}:
\begin{subequations}
\label{eqn:lin-param}
\begin{align}
  \label{eqn:num_epochs}
  \texttt{Number of epochs :} \quad \TotalEpochs \mydefn \log_{2}
  \left(\frac{\nobs (1 - \discount)^2}{8 \log((8 D/\delta) \cdot
    \log\nobs) }\right)
  \end{align}
  \begin{align}
  \label{eqn:Nm}
  \texttt{Recentering sample sizes :} \quad \Avglength_{\epochNum} 
  \mydefn  \base^{\epochNum} \frac{4^2 \cdot \TwiceBaseSqPlusOne^2 \cdot  \log (8 M D/ \delta)}{(1-\discount)^{2}}
  \qquad \text{for } \epochNum = 1, \ldots, \TotalEpochs
  \end{align}
\begin{align}
  \label{eqn:Epoch-samples}
  \texttt{Sample batches:} \quad \text{Partition the } \nobs \text{ samples to obtain }
  \EpochSamples \text{ for }\epiter = 1, \ldots \TotalEpochs
\end{align}
\begin{align}
  \label{eqn:Epoch-length}
  \texttt{Epoch length:} \quad \Klambda = \frac{\nobs}{2 \TotalEpochs}
\end{align}  
\end{subequations}
In the following theorem statement, we use $(c_1, c_2, c_3, c_4)$ to
denote universal constants.
\begin{theorem}
\label{thm:value_function_estimation}
(a) Suppose that the input parameters of
Algorithm~\ref{Algo:variance_reduced_SA} are chosen according to
Eq.~\eqref{eqn:lin-param}.  Furthermore, suppose that the sample
size $\nobs$ satisfies one of the following three stepsize-dependent
lower bounds:
\begin{itemize}
	\item[(a)] $\frac{\nobs}{\TotalEpochs} \geq \unicon_1
          \frac{\log(8 \nobs \Dim/\delta)}{(1 - \discount)^3}$ \text{ for
            recentered linear stepsize } $\tdstep_k = \frac{1}{1 +
            (1 - \discount)k}$,
	\item[(b)] $\frac{\nobs}{\TotalEpochs} \geq \unicon_2 \log (8
          \nobs \Dim / \delta) \cdot \left( \frac{1}{1 - \discount} \right)^{
            \left(\frac{1}{1 - \omega} \vee \frac{2}{\omega} \right)}$
          \text{ for polynomial stepsize} $\tdstep_k =
          \frac{1}{k^\omega}$ with $0 < \omega < 1$, 
  \item[(c)] $\frac{\nobs}{\TotalEpochs} \geq  
   \frac{\unicon_3}{\log \left( \frac{1}{1 - \tdstep(1 - \discount)} \right)}$
          \text{ for constant stepsize} 
          $\tdstep_k = \tdstep \leq \frac{1}{5^2 \cdot 32^2} \cdot \frac{(1 - \discount)^2}{ \log \left( 8 \nobs \Dim / \delta \right)}$.
\end{itemize}
Then for any initilization $\thetabar_1$, the output $\thetabar_{M+1}$ satisfies
\begin{align}
\label{eqn:main-upper-bound}
\|\thetabar_{M+1} - \thetastar \|_\infty & \leq \unicon_4 \cdot
\left\| \widebar{\theta}_{1} - \thetastar\right\|_{\infty} \cdot
\frac{\log^2((8D/\delta) \cdot \log\nobs)}{\nobs^2 (1 - \discount)^4}
\nonumber \\ & \quad + \unicon_4 \cdot \left\{ \LOGDMNSQ \cdot \Big(
\discount \cdot \nu(\Pmat, \thetastar) + \rho(\Pmat, r) \Big) +
\LOGDMN \cdot b(\thetastar) \right\},
\end{align}
with probability exceeding $1 - \delta$.
\end{theorem}
\noindent See Section~\ref{proof:thm:value_function_estimation} for
the proof of this theorem.

\vspace{20pt}
\noindent A few comments on the upper bound provided in
Theorem~\ref{thm:value_function_estimation} are in order. In order to
facilitate a transparent discussion in this section, we use the notation 
$\gtrsim$ in order to denote a relation that holds up to logarithmic factors
in the tuple $\left( \nobs, \Dim, (1 - \discount)^{-1} \right) $.

\paragraph{Initialization dependence:}  The first term
on the right-hand side of the upper bound~\eqref{eqn:main-upper-bound}
depends on the initialization $\thetabar_1$. It should be noted
that when viewed as a function of the sample size $\nobs$, this
initialization-dependent term decays at a faster rate compared to the
other two terms.  This indicates that the performance of
Algorithm~\ref{Algo:variance_reduced_SA} does not depend on the
initialization $\thetabar_1$ in a significant way.  A careful look at
the proof (cf. Section~\ref{proof:thm:value_function_estimation})
reveals that the coefficient of $\| \widebar{\theta}_{1} - \thetastar
\|_{\infty}$ in the bound~\eqref{eqn:main-upper-bound} can be made
significantly smaller.  In particular, for any $\p \geq 1$ the first
term in the right-hand side of bound~\eqref{eqn:main-upper-bound} can
be replaced by
\begin{align*}
  \unicon_4 \cdot \frac{\| \widebar{\theta}_{1} - \thetastar
    \|_{\infty}}{\nobs^\p} \cdot \frac{\log^{\p}((8D/\delta) \cdot
    \log\nobs)}{(1 - \discount)^{2\p}},
\end{align*}
by increasing the recentering sample size~\eqref{eqn:Nm} by a
constant factor and changing the values of the absolute constants
$(\unicon_1, \unicon_2, \unicon_3, \unicon_4)$, with these values
depending only on the value of $\p$.
We have stated and proved a version for $\p =
2$.  Assuming the number of samples $\nobs$ satisfies $\nobs \geq (1 -
\discount)^{-(2 + \Delta)}$ for some $\Delta > 0$, the first term on
the right-hand side of bound~\eqref{eqn:main-upper-bound} can always
be made smaller than the other two terms. In the sequel we
show that each of the lower bound conditions (a)-(c) in the statement
of Theorem ~\ref{thm:value_function_estimation} requires a lower bound
condition $\nobs \gtrsim (1 - \gamma)^{-3}$.


\begin{figure}[t!]
  \begin{center}
    \begin{tabular}{cc}
      \widgraph{0.4\textwidth}{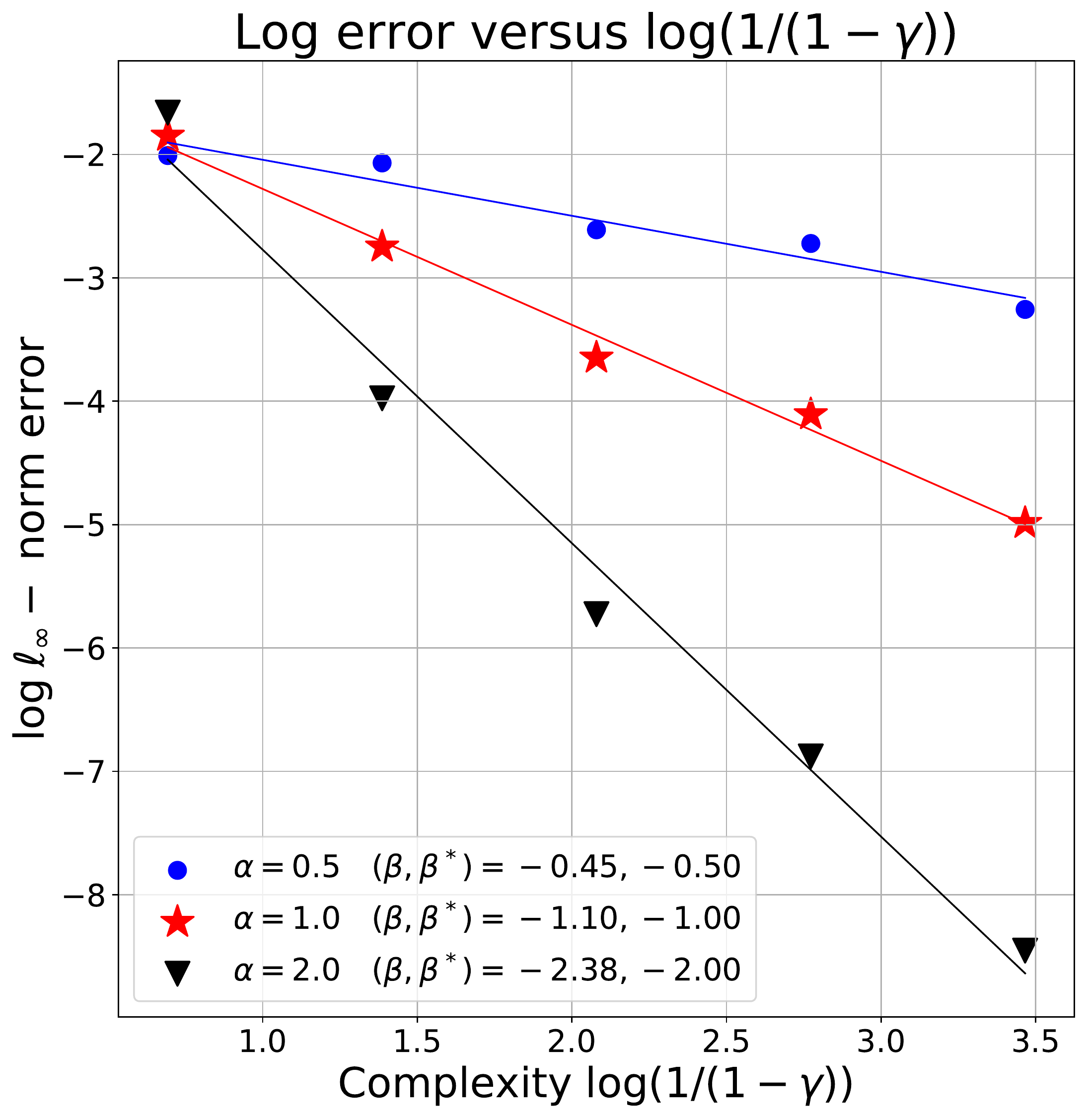} &
      \widgraph{0.4\textwidth}{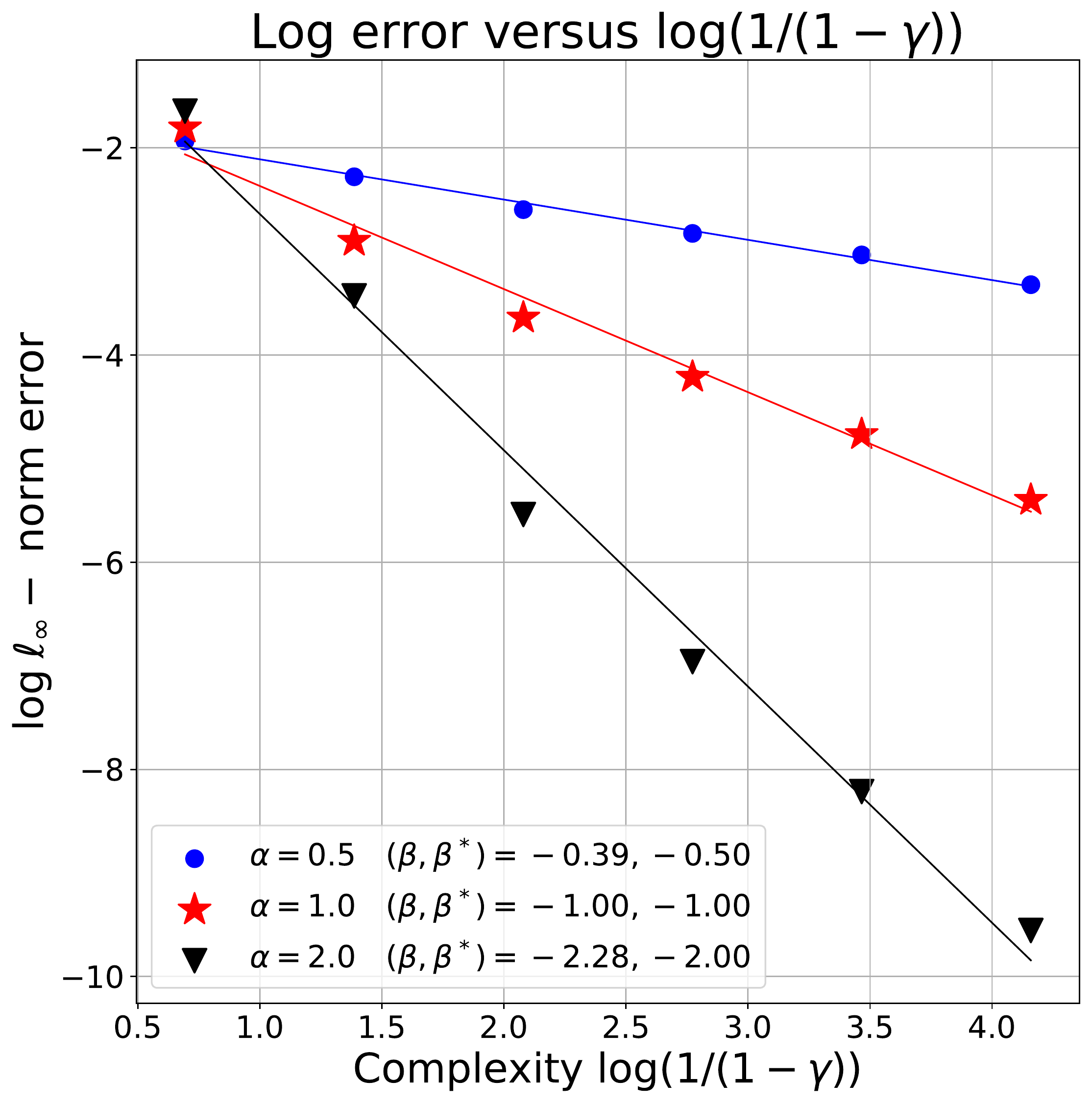} \\ (a)
      Recentered-linear stepsize & (b) Polynomial stepsize
    \end{tabular}
  \caption{Log-log plots of the $\ell_\infty$-error versus the
    discount complexity parameter $1 / (1 - \discount)$ for the VRPE
    algorithm. Each point is computed from an average over $1000$
    trials.  Each trial entails drawing $\nobs = \lceil \frac{8}{(1 -
      \discount)^3} \rceil$ samples from the 2-state MRP in
    Figure~\ref{fig:mrp} with the parameter choices $p =
    \tfrac{4\discount - 1}{3\discount}$, $\lbmaxrew = 1$ and $\tau = 1
    - (1 - \discount)^{\parMRP}$.  Each line on each plot represents a
    different value of $\parMRP$, as labeled in the legend.  We have
    also plotted the least-squares fits through these points, and the
    slopes of these lines are also provided in the legend.  We also
    report the pair $(\betahat, \beta^*)$, where the coefficient
    $\betahat$ denotes the slope of the least-squares fit and
    $\beta^*$ denotes the slope predicted from the lower bound
    calculation~\eqref{eqn:ideal-slope}.  (a) Performance of VRPE for
    the recentered linear stepsize~\eqref{eqn:recentered-linear}. (b)
    Performance of VRPRE with polynomially decaying
    stepsizes~\eqref{eqn:poly_step} with $\omega = 2/3$.}
    \label{fig:VRPE_simulation}
  \end{center}
\end{figure}

\paragraph{Comparing the upper and lower bounds:}
The second and the third terms in~\eqref{eqn:main-upper-bound} show
the instance-dependent nature of the upper bound, and they are the
dominating terms.  Furthermore, assuming that the minimum sample size
requirements from
Theorems~\ref{thm:lower-bound-local-minimax-risk-alt}
and~\ref{thm:value_function_estimation} are met, we find that the
upper bound~\eqref{eqn:main-upper-bound} matches the lower
bound~\eqref{eqn:main-lower-bound} up to logarithmic terms.

It is worthwhile to explicitly compute the minimum sample size
requirements in Theorems~\ref{thm:lower-bound-local-minimax-risk-alt}
and~\ref{thm:value_function_estimation}.  Ignoring the logarithmic
terms and constant factors for the moment, unwrapping the lower bound
conditions (a)-(c) in Theorem~\ref{thm:value_function_estimation}, we
see that for both the constant stepsize and the recentered linear
stepsize the sample size needs to satisfy $\nobs \gtrsim (1 -
\gamma)^{-3}$.  For the polynomial stepsize $\stepsize_{k} =
\frac{1}{k^{\omega}}$, the sample size  has to be at least $(1
- \gamma)^{ -\left(\frac{1}{1 - \omega} \vee
  \frac{2}{\omega}\right)}$. Minimizing the last bound for different
values of $\omega \in (0, 1)$, we see that the minimum value is
attained at $\omega = 2/3$, and in that case the
bound~\eqref{eqn:main-upper-bound} is valid when $\nobs \gtrsim (1 -
\gamma)^{-3}$.  Overall, for all the three stepsize choices discussed
in Theorem~\ref{thm:value_function_estimation} we require $\nobs \gtrsim
(1 - \discount)^{-3}$ in order to certify the upper bound.  Returning
to Theorem~\ref{thm:lower-bound-local-minimax-risk-alt}, from
assumption~\eqref{eqn:N-large-condition} we see that in the best case
scenario, Theorem~\ref{thm:lower-bound-local-minimax-risk-alt} is
valid as soon as $\nobs \gtrsim (1 - \gamma)^{-2}$.  Putting together the
pieces we find that the sample size requirement for
Theorem~\ref{thm:value_function_estimation} is more stringent
than that of
Theorem~\ref{thm:lower-bound-local-minimax-risk-alt}.  Currently
we do not know whether the minimum sample size requirements in
Theorems~\ref{thm:lower-bound-local-minimax-risk-alt}
and~\ref{thm:value_function_estimation} are necessary; answering
this question is an interesting future research direction.

\paragraph{Simulation study:}  It is interesting to demonstrate the sharpness
of our bounds via a simulation study, using the same scheme as our
previous study of TD$(0)$ with averaging.  In
Figure~\ref{fig:VRPE_simulation}, we report the results of this study;
see the figure caption for further details.  At a high level, we see
that the VRPE algorithm, with either the recentered linear stepsize
(panel (a)) or the polynomial stepsize $t^{-2/3}$, produces errors
that decay with the exponents predicted by our instance-dependent
theory for $\parMRP \in \{0.5, 1.0, 2.0 \}$.  See the figure caption
for further details.
 

\section{Proofs}

We now turn to the proofs of our main results.

\subsection{Proof of Proposition~\ref{prop:lower-bound-local-asymp-minimax-risk}}
\label{sec:thm-lower-bound-local-asymp-minimax-risk}

Recall the definition of the matrix $\Sigma_{\Pmat}(\theta)$ from
Eq.~\eqref{eqn:def-Sigma}, and define the covariance matrix
\begin{align}
V_{\problem} = (\IdMat - \gamma \Pmat)^{-1} (\gamma^2
\Sigma_{\Pmat}(\theta) + \sigma_r^2 \IdMat) (\IdMat - \gamma \Pmat)^{-T}.
\end{align}
Recall that we use $Z$ to denote a multivariate
Gaussian random vector \mbox{$Z \sim \normal(0, V_{\problem})$},
and that the sequence $\{\thetatil^{\, \omega}_k\}_{k \geq 1}$
is generated by averaging the iterates of stochastic approximation
with polynomial stepsizes~\eqref{eqn:poly_step} with exponent $\omega$.  
With
this notation, the two claims of the theorem are:
%
\begin{subequations}
\begin{align}
\label{eqn:local-asymp-minimax-risk-result}
\minimax_{\infty}(\problem) & = \E [\norm{Z}_\infty], \quad \mbox{and}
\\
\label{eqn:TD-avg-optimal}
\lim_{\nobs \to \infty} \E \left[\sqrt{\nobs} \cdot
	\| \thetatil^{\, \omega}_{\nobs} - \thetastar \|_{\infty} \right] &= \E [\norm{Z}_\infty].
\end{align}
\end{subequations}

\noindent We now prove each of these claims separately.


\subsubsection{Proof of Eq.~\eqref{eqn:local-asymp-minimax-risk-result}}

For the reader's convenience, let us state a version of the
H\'{a}jek--Le Cam local asymptotic minimax theorem:
\begin{theorem}
\label{thm:Hajek-LeCam}
Let $\{P_{\varthetaprime}\}_{\varthetaprime \in \varTheta}$ be a
family of parametric models, quadratically mean differentiable with
Fisher information matrices $\Fisher_{\varthetaprime}$.  Fix some
parameter $\vartarget \in \varTheta$, and consider a function $\psi:
\varTheta \rightarrow \real^D$ that is differentiable at $\vartarget$.
Then for any quasi-convex loss $L: \R^D \to \R$, we have:
\begin{align}
\lim_{c\to \infty} \lim_{\nobs \to \infty}
\inf_{\hat{\vartheta}_{\nobs}} \sup_{ \substack{\vartheta'
    \\ \norm{\vartheta' - \vartarget}_2 \le c/\sqrt{\Nsamp}}}
\E_{\vartheta'} \Big[L \big(\sqrt{\nobs} \cdot (\hat{\vartheta}_\Nsamp -
  \vartheta') \big) \Big] = \E[L(Z)],
\end{align}
where the infimum is taken over all estimators
$\hat{\vartheta}_{\nobs}$ that are measurable functions of $\nobs$
i.i.d.\ data points drawn from $P_{\vartheta}$, and the
expectation is taken over a multivariate Gaussian $Z \sim \normal(0,
\nabla \psi(\vartarget)^T \Fisher_{\vartarget}^{\dag} \nabla
\psi(\vartarget))$.
\end{theorem}

Returning to the problem at hand, let $\vartheta = (\Pmat, r)$ denote the
unknown parameters of the model and let $\psi(\vartheta) =
\theta(\problem) = (\IdMat - \gamma \Pmat)^{-1}r$ denote the target vector.
A direct application of Theorem~\ref{thm:Hajek-LeCam} shows that
\begin{align}
\label{eqn:Hajek-LeCam-specific}
\minimax_{\infty}(\problem) = \E[\norm{Z}_\infty]~~\text{where $Z =
  \normal(0, \nabla\psi(\vartheta)^T \Fisher_{\vartheta}^{\dag}
  \nabla\psi(\vartheta))$},
\end{align}
where $\Fisher_\vartheta$ is the Fisher information at $\vartheta$.  The
following result provides a more explicit form of the covariance of
$Z$:
\begin{lemma}
\label{lemma:tedious-calculation}
We have the identity
\begin{align}
\label{eqn:tedious-calculation}
\nabla \psi(\vartheta)^T \Fisher_{\vartheta}^{\dag} \nabla
\psi(\vartheta) = (\IdMat - \gamma \Pmat)^{-1} (\gamma^2 \Sigma_{\Pmat}(\theta)
+ \sigma_r^2 \IdMat) (\IdMat - \gamma \Pmat)^{-T}.
\end{align}
\end{lemma}
\noindent Although the proof of this claim is relatively
straightforward, it involves some lengthy and somewhat tedious
calculations; we refer the reader to
Appendix~\ref{sec:detail-calc-thm-one} for the proof. \\

\noindent Given the result from Lemma~\ref{lemma:tedious-calculation},
the claim~\eqref{eqn:local-asymp-minimax-risk-result} follows by
substituting the relation~\eqref{eqn:tedious-calculation} into~\eqref{eqn:Hajek-LeCam-specific}.

\subsubsection{Proof of Eq.~\eqref{eqn:TD-avg-optimal}}

The proof of this claim follows from the results of Polyak and
Juditsky~\cite[Theorem 1]{polyak1992acceleration}, once their
assumptions are verified for TD$(0)$ with polynomial stepsizes. Recall
that the TD iterates in Eq.~\eqref{eq:TD} are given by the sequence
$\{\theta_k \}_{k \geq 1}$, and that $\thetatil^{\, \omega}_k$ denotes
the $k$-th iterate generated by averaging.

For each $k \geq 1$, note the following equivalence between the
notation of our paper and that of Polyak and
Juditsky~\cite{polyak1992acceleration}, or PJ for short:
\begin{align*}
x_k \equiv \theta_k, \qquad \gamma_k \equiv \alpha_k, \qquad \Amat
\equiv \IdMat - \discount \Pmat, \quad \text{ and } \quad \xi_k = (R_k
- r) + (\Zsamp_k - \Pmat) \theta_k.
\end{align*} 
Let us now verify the various assumptions in the PJ
paper. Assumption~2.1 in the PJ paper holds by definition, since the
matrix $\IdMat - \discount \Pmat$ is Hurwitz. Assumption 2.2 in the PJ
paper is also satisfied by the polynomial stepsize sequence for any
exponent $\omega \in (0, 1)$.

It remains to verify the assumptions that must be satisfied by the
noise sequence $\{ \xi_k \}_{k \geq 1}$. In order to do so, write the
$k$-th such iterate as
\begin{align*}
\xi_k = (R_k - r) + (\Zsamp_k - \Pmat) \thetastar + (\Zsamp_k - \Pmat)
(\theta_k - \thetastar).
\end{align*}
Since $\Zsamp_k$ is independent of the sequence $\{\theta_i \}_{i =
  1}^k$, it follows that the condition
\begin{align}
\label{eq:Robbins-Munro}
\lim_{\nobs \to \infty} \E \left[ \| \theta_{\nobs} - \thetastar
  \|^2_2 \right]
\end{align}
suffices to guarantee that Assumptions 2.3--2.5 in the PJ paper are
satisfied. We now claim that for each $\omega \in (1/2, 1]$,
condition~\eqref{eq:Robbins-Munro} is satisfied by the TD
iterates. Taking this claim as given for the moment, note that
applying Theorem 1 of Polyak and
Juditsky~\cite{polyak1992acceleration} establishes
claim~\eqref{eqn:TD-avg-optimal}, for any exponent $\omega \in (1/2,
1)$.

It remains to establish condition~\eqref{eq:Robbins-Munro}. For any $\omega \in (1/2, 1]$, the sequence of stepsizes $\{\alpha_k \}_{k \geq 1}$ satisfies the conditions
\begin{align*}
\sum_{k = 1}^\infty \alpha_k = \infty \quad \text{ and } \quad \sum_{k = 1}^\infty \alpha_k^2 < \infty.
\end{align*}
Consequently, classical results due to Robbins and Munro~\cite[Theorem 2]{robbins1951stochastic} guarantee $\ell^2$-convergence of $\theta_{\nobs}$
to $\thetastar$. 

\subsection{Proof of Theorem~\ref{thm:lower-bound-local-minimax-risk-alt}}
\label{sec:proof-thm-lower-bound-local-minimax-risk}

Throughout the proof, we use the notation $\problem = (\Pmat, r)$ and
$\problem^\prime = (\PmatPrime, r^\prime)$ to denote, respectively, the
problem instance at hand and its alternative.  Moreover, we use
$\theta^* \equiv \theta(\problem)$ and $\theta(\problem')$ to denote
the associated target parameters for each of the two problems
$\problem$ and $\problem'$. We use $\Delta_{\Pmat} = \Pmat -
\PmatPrime$ and $\Delta_r = r - r^\prime$ to denote the differences of
the parameters. For probability distributions, we use $P$ and
${P^\prime}$ to denote the marginal distribution of a single
observation under $\problem$ and $\problem^\prime$, and use $P^\nobs$
and $(P^\prime)^{\nobs}$ to denote the distribution of $\nobs$ i.i.d
observations drawn from $P$ or $P^\prime$, respectively.


\subsubsection{Proof structure}

We introduce two special classes of alternatives of interest, denoted
as $\problemclass_1$ and $\problemclass_2$ respectively:
\begin{align*}
\problemclass_1 = \left\{\problem^\prime = (\PmatPrime, r^\prime)
\mid r^\prime = r\right\}, \quad \mbox{and} \quad \problemclass_2 =
\left\{\problem^\prime = (\PmatPrime, r^\prime) \mid \PmatPrime =
\Pmat\right\}.
\end{align*}
In words, the class $\problemclass_1$ consists of alternatives
$\problem^\prime$ that have the same reward vector $r$ as $\problem$,
but a different transition matrix $\PmatPrime$.  Similarly, the class
$\problemclass_2$ consists of alternatives $\problem^\prime$ with the
same transition matrix $\Pmat$, but a different reward vector.  By
restricting the alternative $\problem^\prime$ within class
$\problemclass_1$ and $\problemclass_2$, we can define
\emph{restricted versions} of the local minimax risk, namely
\begin{subequations}
\begin{align}
\minimax_{\nobs}(\problem; \problemclass_1) \equiv
\sup_{\problem^\prime \in \problemclass_1} \inf_{\hat{\theta}_\nobs}
\max_{\problem \in \{\problem, \problem^\prime\}} \E_{\problem}
\left[ \sqrt{\nobs} \cdot \normbig{\hat{\theta}_{\nobs} - \theta(\problem)}_\infty\right],
\quad \mbox{and} \\
\minimax_{\nobs}(\problem; \problemclass_2) \equiv
\sup_{\problem^\prime \in \problemclass_2} \inf_{\hat{\theta}_\nobs}
\max_{\problem \in \{\problem, \problem^\prime\}} \E_{\problem}
\left[ \sqrt{\nobs} \cdot \normbig{\hat{\theta}_{\nobs} - \theta(\problem)}_\infty\right].
\end{align}
\end{subequations}
The main part of the proof involves showing that there is a universal
constant $c > 0$ such that the lower bounds
\begin{subequations}
\begin{align}
\label{eqn:restricted-lower-bound-one}
\minimax_{\nobs}(\problem; \problemclass_1) &\ge
c \cdot \discount \nu(\Pmat, \thetastar), \quad
\mbox{and} \\
\label{eqn:restricted-lower-bound-two}
\minimax_{\nobs}(\problem; \problemclass_2) & \geq
c \cdot \rho(\Pmat, r)
\end{align}
\end{subequations}
both hold (assuming that the sample size $\nobs$ is sufficiently large
to satisfy the condition~\eqref{eqn:N-large-condition}).  Since we
have $\minimax_{\nobs}(\problem) \ge
\max\left\{\minimax_{\nobs}(\problem; \problemclass_1),
\minimax_{\nobs}(\problem; \problemclass_2)\right\}$, these lower
bounds in conjunction imply the claim
Theorem~\ref{thm:lower-bound-local-minimax-risk-alt}. The next section
shows how to prove these two bounds.


\subsubsection{Proof of the lower bounds~\texorpdfstring{\eqref{eqn:restricted-lower-bound-one}}{19a} 
and\texorpdfstring{~\eqref{eqn:restricted-lower-bound-two}}{19b}:}

Our first step is to lower bound the local minimax risk for each
problem class in terms of a modulus of continuity between the
Hellinger distance and the $\ell_\infty$-norm.
\begin{lemma}
\label{lemma:master-lower-bound-lemma}
For each $\problemclass \in \{\problemclass_1, \problemclass_2\}$, we
have the lower bound $\minimax_\nobs(\problem; \problemclass) \ge
\frac{1}{8} \cdot \underline{\minimax}_\nobs(\problem;
\problemclass)$, where we define
\begin{align}
\label{eqn:testing-lower-bound}
\underline{\minimax}_\nobs(\problem; \problemclass) \defeq
\sup_{\problem^\prime \in
  \problemclass}~\left\{\sqrt{\nobs} \cdot \normbig{\theta(\problem) -
  \theta(\problem^\prime)}_\infty \, \mid \, \dhel(P, P^\prime) \le
\frac{1}{2\sqrt{\nobs}} \right\}.
\end{align}
\end{lemma}
\noindent The proof of Lemma~\ref{lemma:master-lower-bound-lemma}
follows a relatively standard argument, one which reduces estimation
to testing; see
Appendix~\ref{sec:proof-lemma-master-lower-bound-lemma} for details. \\

This lemma allows us to focus our remaining attention on lower
bounding the quantity $\underline{\minimax}_\nobs(\problem;
\problemclass)$.  In order to do so, we need both a lower bound on the
$\ell_\infty$-norm $\normbig{\theta(\problem) -
  \theta(\problem^\prime)}_\infty$ and an upper bound on the Hellinger
distance $\dhel(P, P^\prime)$.  These two types of bounds are provided
in the following two lemmas.  We begin with lower bounds on
the $\ell_\infty$-norm:
\begin{lemma}
\label{lemma:proof-theta-diff-lower-bound}
\begin{enumerate}
\item[(a)] For any $\problem$ and for all $\problem^\prime \in
  \problemclass_1$, we have
  \begin{subequations}
\begin{align}
\label{eqn:theta-diff-lower-bound-one}
\norm{\theta(\problem) - \theta(\problem^\prime)}_\infty \ge \Big(1-
\frac{\discount}{1-\discount} \norm{\Delta_{\Pmat}}_\infty\Big)_+ \cdot
\norm{\discount (\IdMat - \discount \Pmat)^{-1} \Delta_{\Pmat}
  \thetastar}_\infty.
\end{align}
\item[(b)] For any $\problem$ and for all $\problem^\prime \in
  \problemclass_2$, we have
\begin{align}
\label{eqn:theta-diff-lower-bound-two}
\norm{\theta(\problem) - \theta(\problem^\prime)}_\infty \ge 
	\norm{(\IdMat - \discount \Pmat)^{-1} \Delta_{r}}_\infty.
\end{align}
  \end{subequations}
\end{enumerate}
\end{lemma}
\noindent See Appendix~\ref{sec:proof-lemma-proof-theta-diff-lower-bound}
for the proof of this claim. \\

\noindent Next, we require upper bounds on the Hellinger distance:
\begin{lemma}
\label{lemma:proof-dhel-upper-bound} 
\begin{enumerate}
\item[(a)] For each $\problem$ and for all $\problem^\prime \in
  \problemclass_1$, we have
\begin{subequations}
\begin{align}
\label{eqn:dhel-upper-bound-one}
\dhel( P, {P^\prime} )^2 \le \half \sum_{i, j}
\frac{\left((\Delta_{\Pmat})_{i, j}\right)^2}{\Pmat_{i, j}}.
\end{align}
\item[(b)] For each $\problem$ and for all $\problem^\prime \in
  \problemclass_2$, we have
\begin{align}
\label{eqn:dhel-upper-bound-two}
\dhel( P, {P^\prime} )^2 \le \frac{1}{2\sigma_r^2}\norm{r_1 -
  r_2}_2^2.
\end{align}
\end{subequations}
\end{enumerate}
\end{lemma}
\noindent See Appendix~\ref{sec:proof-lemma-proof-dhel-upper-bound}
for the proof of this upper bound. \\

Using Lemmas~\ref{lemma:proof-theta-diff-lower-bound}
and~\ref{lemma:proof-dhel-upper-bound}, we can derive two different
lower bounds.  First, we have the lower bound
$\underline{\minimax}_\nobs(\problem; \problemclass_1) \ge
\underline{\minimax}^\prime_\nobs(\problem; \problemclass_1)$, where
\begin{subequations}
\begin{align}
\label{eqn:lower-bound-one-step-before-final-one}
\underline{\minimax}^\prime_\nobs(\problem; \problemclass_1) \equiv
\sup_{\problem^\prime \in \problemclass_1} \left\{ \sqrt{\nobs} \cdot
\Biggr(1-\frac{\discount \norm{\Delta_\Pmat}_\infty}{1-\discount}
\Biggr)_+ \cdot \norm{\discount (\IdMat - \discount \Pmat)^{-1}
  \Delta_{\Pmat} \thetastar}_\infty \; \mid \; \sum_{i, j}
\tfrac{\left(\left(\Delta_{\Pmat} \right)_{i, j} \right)^2}{\Pmat_{i,
    j}} \le \frac{1}{2\nobs}\right \}.
\end{align}
Second, we have the lower bound $\underline{\minimax}_\nobs(\problem;
\problemclass_2) \ge \underline{\minimax}^\prime_\nobs(\problem;
\problemclass_2)$, where

\begin{align}
\label{eqn:lower-bound-one-step-before-final-two}
\underline{\minimax}^\prime_\nobs(\problem; \problemclass_2) \equiv
\sup_{\problem^\prime \in \problemclass_2}
\left\{ \sqrt{\nobs} \cdot \| \left( \mathbf{I} - \discount \mathbf{P} \right)^{-1} \Delta_{r} \|_{\infty} 
  \;  \;
\frac{1}{\sigma_r^2} \norm{r_1 - r_2}_{2} \le \frac{1}{2 \nobs} \right
\}.
\end{align}
\end{subequations}

In order to complete the proofs of the two lower
bounds~\eqref{eqn:restricted-lower-bound-one}
and~\eqref{eqn:restricted-lower-bound-two}, it suffices to show that
\begin{subequations}
  \begin{align}
\label{eqn:lower-bound-final-two}
\underline{\minimax}^\prime_\nobs(\problem; \problemclass_2) &\ge
\frac{1}{\sqrt{2}} \cdot \rho(\Pmat, r), \quad \mbox{and} \\
\underline{\minimax}^\prime_\nobs(\problem; \problemclass_1) &\ge
\frac{1}{2\sqrt{2}} \cdot \discount \nu(\Pmat, \thetastar)
\label{eqn:lower-bound-final-one}.
\end{align}
\end{subequations}

\paragraph{Proof of the bound~\eqref{eqn:lower-bound-final-two}:}
This lower bound is easy to show---it follows from the definition:
\begin{align*}
\underline{\minimax}^\prime_\nobs(\problem; \problemclass_2) = \frac{\sigma_r}{\sqrt{2}}\norm{(\IdMat - \discount \Pmat)^{-1} \Delta_{r}}_\infty
	= \frac{1}{\sqrt{2}}  \rho(\Pmat, r).
\end{align*}

\paragraph{Proof of the bound~\eqref{eqn:lower-bound-final-one}:}

The proof of this claim is much more delicate.  Our strategy is to
construct a special ``hard'' alternativ, $\hardProblem \in
\problemclass_1$, that leads to a good lower bound on
$\underline{\minimax}^\prime_\nobs(\problem; \problemclass_1)$.
Lemma~\ref{lemma:construct-hardest-one-dimensional-alternatives} below
is the main technical result that we require:
\begin{lemma}
\label{lemma:construct-hardest-one-dimensional-alternatives}
There exists some probability transition matrix $\hardPmat$ with the
following properties:
\begin{enumerate}
\item[(a)] It satisfies the constraint $\sum_{i, j} \frac{\left(
  (\hardPmat - \Pmat)_{i, j} \right)^2}{\Pmat_{i, j}} \le
  \frac{1}{2\nobs}$.
\item[(b)] It satisfies the inequalities
  \begin{align*}
    \norm{\hardPmat - \Pmat}_\infty \le \frac{1}{\sqrt{2 \nobs}},
    \quad \mbox{and} \quad \norm{\discount (\IdMat - \discount \Pmat)^{-1}
      (\hardPmat - \Pmat) \thetastar}_\infty \ge
    \frac{\discount}{\sqrt{2\nobs}} \cdot \nu(\Pmat, \thetastar).
  \end{align*}
\end{enumerate}
\end{lemma}
\noindent See
Appendix~\ref{sec:lemma-construct-hardest-one-dimensional-alternatives}
for the proof of this claim. \\

Given the matrix $\hardPmat$ guaranteed by this lemma, we consider the
``hard'' problem \mbox{$\hardProblem \mydefn (\hardPmat, r) \in
  \problemclass_1$.}  From the definition of
  $\underline{\minimax}^\prime_\nobs(\problem; \problemclass_1)$ in
  Eq.~\eqref{eqn:lower-bound-one-step-before-final-one}, we have
  that
\begin{align*}
\underline{\minimax}^\prime_\nobs(\problem; \problemclass_1) & \ge
\sqrt{\nobs} \cdot \left( 1 - \frac{\discount}{1-\discount} \norm{\Pmat - \hardPmat}_\infty
\right)_+ \cdot \norm{\discount (\IdMat - \discount \hardPmat)^{-1} (\Pmat -
  \hardPmat) \thetastar}_\infty \\
& \ge \sqrt{\nobs} \cdot \left(1-\frac{\discount}{1-\discount}\cdot \frac{1}{\sqrt{2N}}
\right)_+ \cdot \frac{\discount}{\sqrt{2\nobs}} \cdot \nu(\Pmat,
\thetastar) \ge \frac{1}{2\sqrt{2}} \cdot \discount \nu(\Pmat,
\thetastar),
\end{align*}
where the last inequality follows by the assumed lower bound $\nobs \ge
\frac{4\discount^2}{(1-\discount)^2}$. This completes the proof of the lower
bound~\eqref{eqn:lower-bound-final-one}.


\subsection{Proof of Theorem~\ref{thm:value_function_estimation}}
\label{proof:thm:value_function_estimation}

This section is devoted to the proof of
Theorem~\ref{thm:value_function_estimation}, which provides the
achievability results for variance-reduced policy evaluation.

\subsubsection{Proof of part (a):}

We begin with a lemma that characterizes the progress of 
Algorithm~\ref{Algo:variance_reduced_SA} over epochs:

\begin{lems}
\label{lem:key_epoch_recursion}
Under the assumptions of Theorem~\ref{thm:value_function_estimation} (a),
there is an absolute constant $\unicon$ such that for each epoch $\epiter = 1, \ldots, \TotalEpochs$, we have:
\begin{align}
\label{eqn:key_epoch_recursion}
\left\| \thetabar_{m + 1} - \thetastar\right\|_{\infty} & \leq \frac{
  \| \thetabar_\epiter - \thetastar \|_\infty }{\baseSq} \notag \\ & +
\unicon \left\lbrace \sqrt{\LOGDMNm} \Big( \discount \cdot
\nu(\Pmat, \thetastar) + \rho(\Pmat, r) \Big) + \LOGDMNm \cdot
b(\thetastar) \right\rbrace,
\end{align}
with probability exceeding $1 - \frac{\pardelta}{\TotalEpochs}$.
\end{lems}

\vspace{10pt}
Taking this lemma as given for the moment, let us complete the
proof. We use the shorthand
\begin{align}
  \label{eq:tau-short}
\tauOne_\epiter \mydefn \sqrt{\LOGDMNm} \Big( \discount \cdot
\nu(\Pmat, \thetastar) \rho(\Pmat, r) \Big) \quad \text{and}\quad
\tauTwo_\epiter \mydefn \LOGDMNm \cdot b(\thetastar)
\end{align}
to ease notation, and note that $\frac{\tauOne_{m}}{\baseSqrt} \leq
\tauOne_{m+1}$ and $ \frac{\tauTwo_{m}}{\base} \leq \tauTwo_{m + 1}$,
for each $m \geq 1$. Using this notation and unwrapping the recursion
relation from Lemma~\ref{lem:key_epoch_recursion}, we have
\begin{align*}	
\left\|\bar{\theta}_{M + 1}-\thetastar \right\|_{\infty} & \leq
\frac{\left\| \widebar{\theta}_{\NumEpoch} -
  \thetastar\right\|_{\infty}}{\baseSq} + \unicon (\tauOne_\NumEpoch +
\tauTwo_M) \\ & \stackrel{(i)}{\leq} \frac{ \left\|
  \widebar{\theta}_{M - 1} - \thetastar\right\|_{\infty} }{\baseSq^2} +
\frac{\unicon}{\base} \left( \tauOne_{M} + \tauTwo_{M} \right) +
\unicon (\tauOne_M + \tauTwo_M) \\ & \stackrel{(ii)}{\leq} \frac{
  \left\| \widebar{\theta}_{1} - \thetastar\right\|_{\infty} }{\baseSq^M} +
\base \unicon (\tauOne_M + \tauTwo_M) .
\end{align*}
Here, step (i) follows by applying the one-step application of the recursion~\eqref{eqn:key_epoch_recursion},
and by using the upper bounds $\frac{\tauOne_{m}}{\baseSqrt} \leq \tauOne_{m+1}$  
and $ \frac{\tauTwo_{m}}{\base} \leq \tauTwo_{m + 1}$. Step (ii) follows by repeated application of the recursion~\eqref{eqn:key_epoch_recursion}. The last inequality holds with probability at least $1 - \pardelta$ by a 
union bound over the $\TotalEpochs$ epochs. 

It remains to express the quantities $\baseSq^\TotalEpochs$,
$\tauOne_M$ and $\tauTwo_M$---all of which are controlled by the
recentering sample size $\nobs_{\TotalEpochs}$---in terms of the
total number of available samples $\nobs$. Towards this end, observe
that the total number of samples used for recentering at
$\TotalEpochs$ epochs is given by
\begin{align*}
 \sum_{m = 1}^{M} \Nm \asymp \base^{\TotalEpochs} \cdot
 \frac{\log(8MD/\delta)}{(1 - \discount)^2}.
\end{align*}
Substituting the value of $\TotalEpochs = \log_\base \left(\frac{\nobs
  (1 - \discount)^2}{8 \log((8 D/\delta) \cdot \log\nobs) }\right)$ we
have
\begin{align*}
 \unicon_1 \nobs \leq \NM \asymp \sum_{m = 1}^{M} \Nm \leq
 \frac{N}{2},
\end{align*}
where $\unicon_1$ is a universal constant. Consequently, the total
number of samples used by Algorithm~\ref{Algo:variance_reduced_SA} is
given by
\begin{align*}
  \TotalEpochs \Klambda + \sum_{m = 1}^{M} \Nm \leq \frac{N}{2} +
  \frac{\nobs}{2} = \nobs,
\end{align*}  
where in the last equation we have used the fact that $\TotalEpochs
\Klambda = \frac{\nobs}{2}$.  Finally, using ${\TotalEpochs =
  \log_\base \left(\frac{\nobs (1 - \discount)^2}{8 \log((8 D/\delta)
    \cdot \log \nobs) }\right)}$ we have the following relation for
some universal constant $\unicon$:
\begin{align*}
  \baseSq^\TotalEpochs = \unicon \cdot \frac{\nobs^2 (1 -
    \discount)^4}{ \log^2((8D/\delta) \cdot \log \nobs) }
\end{align*}
Putting together the pieces, we conclude that
\begin{align*}
  \|\thetabar_{M + 1} - \thetastar \|_\infty & \leq \unicon_2 \left\|
  \widebar{\theta}_{1} - \thetastar\right\|_{\infty} \cdot
  \frac{\log^2((8D/\delta) \cdot \log\nobs )}{\nobs^2 (1 -
    \discount)^4} \nonumber \\ &\quad + \unicon_2 \left\{
  \sqrt{\LOGDMN} \Big( \discount \cdot \nu(\Pmat, \thetastar) +
  \rho(\Pmat, r) \Big) + \LOGDMN \cdot b(\thetastar) \right\},
\end{align*}
for a suitable universal constant $\unicon_2$.  The last bound is
valid with probability exceeding $1 - \delta$ via the union bound.  In
order to complete the proof, it remains to prove
Lemma~\ref{lem:key_epoch_recursion}, which we do in the following
subsection.


\subsubsection{Proof of Lemma~\ref{lem:key_epoch_recursion}}

We now turn to the proof of the key lemma within the argument.  We
begin with a high-level overview in order to provide intuition.  In the
$m$-th epoch that updates the estimate from $\thetabar_{m}$ to
$\thetabar_{m+1}$, the vector $\thetabar \equiv \thetabar_{m}$ is used
to recenter the updates.
Our analysis of the $m$-th epoch is based on a sequence of recentered
operators $\{\Jstar_{k}^{m}\}_{k \geq 1}$ and their population analogs
$\Jstar^m(\theta)$. The action of these operators on a point $\theta$
is given by the relations
\begin{subequations}
\begin{align}
\label{eqn:perturbed-operators}
\Jstar_{k}^{m}(\theta):= \That_{k}(\theta)- \That_{k}(\thetabar_m) +
\Ttil_{\nobs}(\thetabar_m), \quad \text { and } \quad
\Jstar^m(\theta):=\Tstar(\theta)-\Tstar(\thetabar_m)+\Ttil_{\nobs}(\thetabar_m).
\end{align}
By definition, the updates within epoch $m$ can be written as
\begin{align}
\label{eqn:perturbed-update-equation}
  \theta_{k+1}=\left(1-\tdstep_{k}\right) \theta_{k}+\tdstep_{k}
  \Jstar_{k}^m\left(\theta_{k}\right).
\end{align}
Note that the operator $\Jstar^m$ is $\discount$-contractive in $\|
\cdot \|_\infty$-norm, and as a result it has a unique fixed point,
which we denote by $\thetahat_m$. Since
$\Jstar^m(\theta)=\mathbb{E}\left[\Jstar_{k}^m(\theta)\right]$ by
construction, when studying epoch $m$, it is natural to analyze the
convergence of the sequence $\{\theta_{k}\}_{k \geq 1}$ to
$\thetahat_m$.

Suppose that we have taken $\epochLen$ steps within epoch
$m$. Applying the triangle inequality yields the bound
\begin{align}
\label{eqn:theta_m-triangle-inequality}
\|\thetabar_{m + 1} - \thetastar \|_\infty =
\left\|\theta_{\epochLen+1}-\thetastar\right\|_{\infty} \leq
\left\|\theta_{\epochLen+1}-\thetahat_m \right\|_{\infty} +
\left\|\thetahat_m - \thetastar\right\|_{\infty}.
\end{align}
\end{subequations}
With this decomposition, our proof of
Lemma~\ref{lem:key_epoch_recursion} is based on two auxiliary lemmas
that provide high-probability upper bounds on the two terms on the
right-hand side of inequality~\eqref{eqn:theta_m-triangle-inequality}.

\begin{lems}
\label{lem:epoch-error-lemma}
Let $(c_1, c_2, c_3)$ be positive numerical constants, and suppose
that the epoch length $\Klambda$ satisfies one the following three
stepsize-dependent lower bounds:
\begin{itemize}
  \item[(a)] $\Klambda \geq \unicon_1
          \frac{\log(8 K M D/\delta)}{(1 - \discount)^3}$ \text{ for
            recentered linear stepsize } $\tdstep_k = \frac{1}{1 +
            (1 - \discount)k}$,
  \item[(b)] $ \Klambda \geq \unicon_2 \log (8
         K M D / \delta) \cdot \left( \frac{1}{1 - \discount} \right)^{
            \left(\frac{1}{1 - \omega} \vee \frac{2}{\omega} \right)}$
          \text{ for polynomial stepsize} $\tdstep_k =
          \frac{1}{k^\omega}$ with $0 < \omega < 1$,
  \item[(c)] $ \Klambda \geq  
   \frac{\unicon_3}{\log \left( \frac{1}{1 - \tdstep(1 - \discount)} \right)}$
          \text{ for constant stepsize} 
          $\tdstep_k = \tdstep \leq \frac{(1 - \discount)^2}{ \log \left( 8 \Klambda \TotalEpochs \Dim / \delta \right)}
  \cdot \frac{1}{5^2 \cdot 32^2} $.
\end{itemize}
Then after $\Klambda$ update steps with epoch $\epochNum$, the iterate
$\theta_{\Klambda + 1}$ satisfies the bound
  \begin{align}
    \|\theta_{\epochLen+1}-\thetahat_m \|_{\infty} \leq
    \frac{1}{\TwiceBaseSq} \| \thetabar_m - \thetastar \|_\infty +
    \frac{1}{\TwiceBaseSq} \|\thetahat_m - \thetastar \|_{\infty}
    \quad \mbox{with probability at least $1 -
      \frac{\delta}{2\TotalEpochs}$.}
  \end{align}
\end{lems}
\noindent
See Appendix~\ref{proof:lem:epoch-error-lemma} for the proof of this
claim. \\

\noindent Our next auxiliary result provides a high-probability bound
on the difference $\| \thetahat_m - \thetastar \|_{\infty}$.
\begin{lems}
\label{lem:epoch-deviation-lemma}
There is an absolute constant $\unicon_4$ such that for any
recentering sample size satisfying $\Avglength_{\iterM} \geq 4^2
\cdot \TwiceBaseSqPlusOne^2 \cdot \frac{\log (M D/
  \delta)}{(1-\discount)^{2}}$, we have
\begin{align*}
  \| \thetahat_m - \thetastar \|_{\infty} & \leq
  \tfrac{1}{\TwiceBaseSqPlusOne} \| \thetabar_m - \thetastar \|_\infty
  + \unicon_4 \left\{ \sqrt{\LOGDMNm} \Big( \discount \cdot \nu(\Pmat,
  \thetastar) + \rho(\Pmat, r) \Big) + \LOGDMNm \cdot b(\thetastar)
  \right\},
\end{align*}
with probability exceeding $1 - \frac{\delta}{2M}$.
\end{lems}
\noindent See Appendix~\ref{proof:lem:epoch-deviation-lemma} for the
proof of this claim. \\

\vspace{10pt}

\noindent With Lemmas~\ref{lem:epoch-error-lemma}
and~\ref{lem:epoch-deviation-lemma} in hand, the remainder of the
proof is straightforward. Recall from Eq.~\eqref{eq:tau-short}
the shorthand notation $\tauOne_m$ and $\tauTwo_m$.  Using our earlier
bound~\eqref{eqn:theta_m-triangle-inequality}, we have that at the end
of epoch $m$ (which is also the starting point of epoch $m + 1$),
\begin{align*}
\left\| \thetabar_{m + 1} - \thetastar\right\|_{\infty} & \leq \|
\theta_{\epochLen + 1} - \thetahat_m \|_\infty + \|
\thetahat_m - \thetastar \|_\infty \\
& \stackrel{(i)}{\leq}
\left\{ \frac{\| \thetabar_m - \thetastar \|_\infty}{\TwiceBaseSq} +
\frac{1}{\TwiceBaseSq}\left\|\thetahat_m - \thetastar
\right\|_{\infty}\right\} + \left\|\thetahat_m -
\thetastar\right\|_{\infty} \notag \\
& = \frac{ \| \thetabar_m - \thetastar \|_\infty }{\TwiceBaseSq} + \frac{\TwiceBaseSqPlusOne}{\TwiceBaseSq}
\cdot \left\|\thetahat_m - \thetastar\right\|_{\infty} \\
& \stackrel{(ii)}{\leq} \frac{ \| \thetabar_m - \thetastar \|_\infty
}{\TwiceBaseSq} + \frac{1}{\TwiceBaseSq} \left\lbrace \| \thetabar_m - \thetastar
\|_\infty + \unicon_4 (\tauOne_m + \tauTwo_m) \right\rbrace \\
& \leq \frac{ \| \thetabar_m - \thetastar \|_\infty }{\baseSq} + \unicon_4
(\tauOne_m + \tauTwo_m),
\end{align*}
where inequality (i) follows from Lemma~\ref{lem:epoch-error-lemma}(a), and inequality (ii) from Lemma~\ref{lem:epoch-deviation-lemma}. Finally, the sequence of inequalities above holds with probability at least $1-\frac{\pardelta}{\TotalEpochs}$ via a union bound. This completes the proof of Lemma~\ref{lem:key_epoch_recursion}.

\subsubsection{Proof of Theorem~\ref{thm:value_function_estimation}, parts (b) and (c)}
\label{proof:cor:VR_SA_poly}

The proofs of Theorem~\ref{thm:value_function_estimation} parts (b) and (c) 
require versions of Lemma~\ref{lem:key_epoch_recursion} for the 
polynomial stepsize~\eqref{eqn:poly_step} and constant stepsize~\eqref{eqn:constant_step}, respectively.
These two versions of Lemma~\ref{lem:key_epoch_recursion}
can be obtained by simply replacing Lemma~\ref{lem:epoch-error-lemma},
part (a), by Lemma~\ref{lem:epoch-error-lemma},
parts (b) and (c), respectively, in the proof of Lemma~\ref{lem:key_epoch_recursion}.


\section{Discussion}
\label{SecDiscussion}

We have discussed the problem of policy evaluation in
discounted Markov decision processes.  Our contribution is
three-fold.  First, we provided a non-asymptotic instance-dependent
local-minimax bound on the $\ell_{\infty}$-error for the policy
evaluation problem under the generative model.  
Next, via careful simulations, we showed that the
standard TD-learning algorithm---even when combined with
Polyak-Rupert iterate averaging---does not yield ideal
non-asymptotic behavior as captured by our lower bound.  In order to
remedy this difficulty, we introduced and analyzed a variance-reduced version of the standard TD-learning algorithm which achieves our
non-asymptotic instance-dependent lower bound up to logarithmic
factors.  Both the upper and lower bounds discussed in this paper
hold when the sample size is bigger than an explicit threshold;
relaxing this minimum sample size requirement is an interesting future
research direction. Finally, we point out that although we have
focused on the tabular policy evaluation problem, the variance-reduced
algorithm discussed in this paper can be applied in more generality,
and it would be interesting to explore applications of this algorithm to non-tabular settings.

\subsection*{Acknowledgements}

This research was partially supported by Office of Naval Research
Grant DOD ONR-N00014-18-1-2640 and National Science Foundation Grant
DMS-1612948 to MJW and by a grant from the DARPA Program on Lifelong Learning
Machines to MIJ.

\bibliographystyle{abbrv} \bibliography{koulikpaperTD-arXiv,MRPs-arXiv}

\appendix



\section{Proofs of auxiliary lemmas for
  Proposition~\ref{prop:lower-bound-local-asymp-minimax-risk}}

In this section, we provide proofs of the auxiliary lemmas that
underlie the proof of
Proposition~\ref{prop:lower-bound-local-asymp-minimax-risk}.


\subsection{Proof of Lemma~\ref{lemma:tedious-calculation}}
\label{sec:detail-calc-thm-one}

The proof is basically a lengthy computation. For clarity, let us
decompose the procedure into three steps.  In the first step, we
compute an explicit form for the inverse information matrix
$\Fisher_{\vartheta}^{\dag}$. In the second step, we evaluate the gradient
$\grad \psi(\vartheta)$. In the third and final step, we use the
result in the previous two steps to prove the
claim~\eqref{eqn:tedious-calculation} of the lemma.

\paragraph{Step 1:} In the first step, we evaluate $\Fisher_{\vartheta}^{\dag}$.
Recall that our data $(\Zsamp, R)$ is generated as follows. We
generate the matrix $\Zsamp$ and the vector $R$ independently. Each
row of $\Zsamp$ is generated independently. Its $j$th row, denoted by
$z_j$, is sampled from a multinomial distribution with parameter
$p_j$, where $p_j$ denotes the $j$th row of $\Pmat$. The vector $R$ is
sampled from $\normal(r, \sigma_r^2 \IdMat)$.  Because of this
independence structure, the Fisher information $\Fisher_{\vartheta}$
is a block diagonal matrix of the form
\begin{align*}
\Fisher_{\vartheta} = \begin{bmatrix}
	\Fisher_{p_1} & 0& 0 \ldots & 0 & 0 \\
	0 & \Fisher_{p_2} & 0 \ldots & 0  & 0 \\
	0 & 0 & \ldots & 0 & 0  \\
	0 & 0 & 0 \ldots & \Fisher_{p_D} & 0 \\
	0 & 0 & 0 \ldots & 0 & \Fisher_r 
\end{bmatrix}.
\end{align*}
Here each sub-block matrix $\Fisher_{p_j}$ is the Fisher information
corresponding to the model where a single data $Z_j$ is sampled from
the multinomial distribution with parameter $p_j$, and $\Fisher_r$ is
the Fisher information corresponding to the model in which a single data
point $R$ is sampled from $\normal(r, \sigma_r^2 \IdMat)$. Thus, the inverse
Fisher information $\Fisher_{\vartheta}^{\dag}$ is also a block
diagonal matrix of the form
\begin{align}
\label{eqn:I-block}
\Fisher_{\vartheta}^{\dag} = \begin{bmatrix} \Fisher_{p_1}^{\dag} & 0& 0 \ldots &
  0 & 0 \\ 0 & \Fisher_{p_2}^{\dag} & 0 \ldots & 0 & 0 \\ 0 & 0 & \ldots & 0
  & 0 \\ 0 & 0 & 0 \ldots & \Fisher_{p_D}^{\dag} & 0 \\ 0 & 0 & 0 \ldots & 0
  & \Fisher_r^{\dag}
\end{bmatrix}.
\end{align}
It is easy to compute $\Fisher_{p_j}^{\dag}$ and $\Fisher_r^{\dag}$:
\begin{subequations}
\begin{align}
\label{eqn:I-p-j}
\Fisher_{p_j}^{\dag} & = \diag(p_j) - p_j p_j^T = \cov(Z_j - p_j)
\qquad \text{for $j \in [D]$, and} \\
\label{eqn:I-r}
\Fisher_r^{\dag} & = \Fisher_r^{-1} = \sigma_r^2 I.
\end{align}
\end{subequations}
For a vector $q \in \R^D$, we use $\diag(q) \in \R^{D\times D}$ to
denote the diagonal matrix with diagonal entries $q_j$.

\paragraph{Step 2:} In the second step, we evaluate $\grad \psi(\vartheta)$.
Recall that $\psi(\vartheta) = (\IdMat - \gamma \Pmat)^{-1}r$. It is
straightforward to see that
\begin{align}
\label{eqn:grad-r}
\nabla_r\psi(\vartheta) = (\IdMat - \gamma \Pmat)^{-1}.
\end{align}
Below we evaluate $\nabla_{p_j} \psi(\theta)$ for $j \in [D]$, where $p_j$ is the $j$th row of the matrix $\Pmat$.
We show that
\begin{align}
\label{eqn:grad-p-j}
\nabla_{p_j}\psi(\vartheta) = \gamma (\IdMat - \gamma \Pmat)^{-1} e_j \theta^T.
\end{align}
Here we recall $\theta = \psi(\vartheta) = (\IdMat - \gamma \Pmat)^{-1}r$.

To prove Eq.~\eqref{eqn:grad-p-j}, we start with the following
elementary fact: for the matrix inverse mapping $A \to A^{-1}$, we
have $\frac{\partial A^{-1}}{\partial A_{jk}} = -A^{-1} e_j e_k^T
A^{-1}$ for all $j, k \in [D]$.  Combining this fact with chain rule,
we find that
\begin{align*}
\frac{\partial \psi(\vartheta)}{\partial \Pmat_{jk}} &= \gamma (\IdMat -
\gamma \Pmat)^{-1} e_j e_k^T (\IdMat - \gamma \Pmat)^{-1}r = \gamma (\IdMat - \gamma
\Pmat)^{-1} e_j \theta^T e_k,
\end{align*}
valid for all $j, k \in [\Dim]$.  This immediately implies
Eq.~\eqref{eqn:grad-p-j} since $p_j$ is the vector with
coordinates $\Pmat_{jk}$.

\paragraph{Step 3:} In the third step, we evaluate
$\grad \psi(\vartheta)^T \Fisher_{\vartheta}^{\dag} \grad
\psi(\vartheta)$.  From Eq.~\eqref{eqn:I-block}, we observe that
the inverse Fisher information $\Fisher_{\vartheta}^{\dag}$ has a
block structure.  Consequently, we can write
\begin{align}
\label{eqn:compute-block}
\grad \psi(\vartheta)^T \Fisher_{\vartheta}^{\dag} \grad
\psi(\vartheta) = \sum_{j \in [D]} \grad_{p_j} \psi(\vartheta)^T
\Fisher_{p_j}^{\dag} \grad_{p_j} \psi(\vartheta) + \grad_{R}
\psi(\vartheta)^T \Fisher_R^{\dag} \grad_{R} \psi(\vartheta).
\end{align}
Combining Eqs.~\eqref{eqn:I-r} and~\eqref{eqn:grad-r} yields
\begin{align}
\label{eqn:compute-term-one}
\grad_{R} \psi(\vartheta)^T \Fisher_R^{\dag} \grad_{R} \psi(\vartheta)
= \sigma_r^2 (\IdMat - \gamma \Pmat)^{-1}(\IdMat - \gamma \Pmat)^{-T}.
\end{align}
Combining Eqs.~\eqref{eqn:I-p-j} and~\eqref{eqn:grad-p-j} yields
\begin{align*}
\grad_{p_j} \psi(\vartheta)^T \Fisher_{p_j}^{\dag} \grad_{p_j}
\psi(\vartheta) &= \gamma^2 (\IdMat - \gamma \Pmat)^{-1} e_j \theta^T \cov(Z_j
- p_j) \theta e_j^T(\IdMat - \gamma \Pmat)^{-T} \\ &= \gamma^2 (\IdMat - \gamma
\Pmat)^{-1} e_j \cov((Z_j-p_j)^T \theta) e_j^T (\IdMat - \gamma \Pmat)^{-T},
\end{align*}
valid for each $j \in [D]$. Summing over $j \in [D]$ then leads to
\begin{align}
\sum_{j \in [D]} \grad_{p_j} \psi(\vartheta)^T \Fisher_{p_j}^{\dag}
\grad_{p_j} \psi(\vartheta) &= \gamma^2 (\IdMat - \gamma \Pmat)^{-1}
\Big(\sum_{j \in [D]} e_j \cov((Z_j-p_j)^T \theta) e_j^T\Big) (\IdMat -
\gamma \Pmat)^{-T} \nonumber \\
\label{eqn:compute-term-two}
& = \gamma^2 (\IdMat - \gamma \Pmat)^{-1} \Sigma_{\Pmat}(\theta) (\IdMat - \gamma
\Pmat)^{-T},
\end{align}
where the last line uses the definition of $\Sigma_{\Pmat}(\theta)$ in
Eq.~\eqref{eqn:def-Sigma}.  Finally, substituting
Eq.~\eqref{eqn:compute-term-one} and
Eq.~\eqref{eqn:compute-term-two} into
Eq.~\eqref{eqn:compute-block} yields the
claim~\eqref{eqn:tedious-calculation}, which completes the proof of
Lemma~\ref{lemma:tedious-calculation}.


\section{Proofs of auxiliary lemmas for Theorem~\ref{thm:lower-bound-local-minimax-risk-alt}}

In this appendix, we detailed proofs of the auxiliary lemmas that
underlie the proof of the non-asymptotic local minimax lower bound
stated in Theorem~\ref{thm:lower-bound-local-minimax-risk-alt}.


\subsection{Proof of Lemma~\ref{lemma:master-lower-bound-lemma}}
\label{sec:proof-lemma-master-lower-bound-lemma}

The proof uses the standard device of reducing estimation to testing
(see, e.g.,~\cite{Birge83, Tsybakov09,wainwright2019high}).  The first step
is to lower bound the minimax risk over $\problem$ and
$\problem^\prime$ by its averaged risk:
\begin{align}
\inf_{\hat{\theta}_\nobs} \max_{\problem \in \{\problem,
  \problem^\prime\}} \E_{\problem}
\left[\normbig{\theta-\theta(\problem)}_\infty\right] \ge \half
\left(\E_{P^\nobs}\left[\normbig{\hat{\theta}_{\nobs} -
    \theta}_{\infty}\right] +
\E_{{P^\prime}^\nobs}\left[\normbig{\hat{\theta}_{\nobs} -
    \theta^\prime}_{\infty}\right] \right).
\end{align}
By Markov's inequality, for any $\delta \ge 0$, we have
\begin{align*}
\E_{P^\nobs}\left[\normbig{\hat{\theta}_{\nobs} -
    \theta}_{\infty}\right] +
\E_{{P^\prime}^\nobs}\left[\normbig{\hat{\theta}_{\nobs} -
    \theta^\prime}_{\infty}\right] \ge \delta
\left[P^\nobs\left(\normbig{\hat{\theta}_{\nobs} - \theta}_{\infty}
  \ge \delta\right) +
      {P^\prime}^\nobs\left(\normbig{\hat{\theta}_{\nobs} -
        \theta^\prime}_{\infty} \ge \delta\right)\right].
\end{align*}
If we define $\delta_{01} \defeq \half \norm{\theta -
  \theta^\prime}_\infty$, then we have the implication
\begin{align}
\label{eqn:mutual-exclusion}
\norm{\theta-\theta}_\infty < \delta_{01} \quad \Longrightarrow \quad
\norm{\theta - \theta^\prime}_\infty > \delta_{01},
\end{align}
from which it follows that
\begin{align*}
\E_{P^n}\left[\normbig{\hat{\theta}_{n} - \theta}_{\infty}\right] +
\E_{{P^\prime}^n}\left[\normbig{\hat{\theta}_{n} -
    \theta^\prime}_{\infty}\right] &\ge \delta_{01} \left[1 -
  P^n(\normbig{\hat{\theta}_{n} - \theta}_{\infty} < \delta_{01}) +
  {P^\prime}^n(\normbig{\hat{\theta}_{n} - \theta^\prime}_{\infty} \ge
  \delta_{01})\right] \\ &\ge \delta_{01} \left[1 -
  P^n(\normbig{\hat{\theta}_\nobs- \theta^\prime}_\infty \ge
  \delta_{01}) + {P^\prime}^n(\normbig{\hat{\theta}_\nobs-
    \theta^\prime}_\infty \ge \delta_{01})\right] \\ &\ge \delta_{01}
\left[1 - \normbig{P^n - {P^\prime}^n}_{\tv}\right] \ge \delta_{01}
\left[1 - \sqrt{2}\dhel(P^n, {P^\prime}^n)^2\right].
\end{align*}
The tensorization property of Hellinger distance (cf.\ Section 15.1 in~\cite{wainwright2019high}) guarantees that
\begin{align*}
\dhel(P^\nobs,{P^\prime}^\nobs)^2 = 1- \left(1- \dhel(P,
     {P^\prime})^2\right)^\nobs \le \nobs \; {\dhel(P, P^\prime)}^2.
\end{align*}
Thus, we have proved that
\begin{align*}
\inf_{\hat{\theta}_\nobs} \max_{\qproblem \in \{\problem,
  \problem^\prime\}} \E_{\qproblem}
\left[\normbig{\theta-\theta(\qproblem)}_\infty\right] \ge \frac{1}{4}
\normbig{\theta(\problem) - \theta(\problem^\prime)}_\infty \cdot
\left(1 - \sqrt{2} \nobs \cdot{\dhel(P, P^\prime)}^2\right)_+.
\end{align*}
Taking the supremum over all the possible alternatives
$\problem^\prime \in \problemclass$ yields
\begin{align}
\minimax_\nobs(\problem; \problemclass) &\ge \sup_{\problem^\prime \in
  \problemclass}~ \frac{1}{4} \cdot \sqrt{\nobs}\normbig{\theta(\problem) -
  \theta(\problem^\prime)}_\infty \cdot \left(1 - \sqrt{2} \nobs
\cdot{\dhel(P, P^\prime)}^2\right)_+.
\end{align}
A calculation shows that this bound implies the claim in
Lemma~\ref{lemma:master-lower-bound-lemma}.


\subsection{Proof of Lemma~\ref{lemma:proof-theta-diff-lower-bound}}
\label{sec:proof-lemma-proof-theta-diff-lower-bound}

Recall the shorthand $\DeltaPPrime = \Pmat - \PmatPrime$ and $\Delta_r
= r - r^\prime$, and let $\thetastar \equiv \theta(\problem)$. We
prove that $ \norm{\theta(\problem) - \theta(\problem^\prime)}_\infty$
is lower bounded by
\begin{multline}
\label{eqn:theta-diff-lower-bound}
  \norm{\discount (\IdMat - \discount \Pmat)^{-1} \DeltaPPrime
    \thetastar + (\IdMat - \discount \Pmat)^{-1}\Delta_r }_\infty -
  \left(\frac{\discount \norm{\DeltaPPrime}_{\infty}}{(1-\discount)}
  \norm{\discount (\IdMat - \discount \Pmat)^{-1} \DeltaPPrime
    \thetastar}_\infty + \frac{\discount
    \norm{\DeltaPPrime}_\infty\norm{\Delta_r}_\infty}{(1-\discount)^2}
  \right).
\end{multline}
Since $\theta(\problem) = (\IdMat - \discount \Pmat)^{-1}r$ and
$\theta(\problem^\prime) = (\IdMat - \discount
\PmatPrime)^{-1}r^\prime$ by definition, if we introduce the shorthand
$\MatDiff{\Pmat} = (\IdMat - \discount \Pmat)^{-1} - (\IdMat -
\discount \PmatPrime)^{-1}$, some elementary calculation gives the
identity
\begin{align}
\label{eqn:represent-theta-0-theta-1}
\theta(\problem) - \theta(\problem^\prime) = \MatDiff{\Pmat}r +
(\IdMat - \discount \Pmat)^{-1}\Delta_r +
\MatDiff{\Pmat}\Delta_r.
\end{align}
Now we find a new expression for $\MatDiff{\Pmat} = (\IdMat -
\discount \Pmat)^{-1} - (\IdMat - \discount \PmatPrime)^{-1}$ that
is easy to control.  Recall the elementary identity $A_1^{-1} =
A_0^{-1} + A_1^{-1}(A_0 - A_1)A_0^{-1}$ for any matrices $A_0,
A_1$. Thus,
\begin{align*}
\begin{split}
\MatDiff{\Pmat}&=(\IdMat - \discount \Pmat)^{-1} - (\IdMat - \discount \PmatPrime)^{-1}\\
	&= \discount (\IdMat - \discount \PmatPrime)^{-1} (\Pmat - \PmatPrime) (\IdMat - \discount \Pmat)^{-1} \\
	&= \discount (\IdMat - \discount \Pmat)^{-1} (\Pmat - \PmatPrime) (\IdMat - \discount \Pmat)^{-1} 
		+ \discount^2 (\IdMat - \discount \PmatPrime)^{-1} (\Pmat - \PmatPrime) (\IdMat - \discount \Pmat)^{-1} (\Pmat - \PmatPrime) (\IdMat - \discount \Pmat)^{-1} \\	
	&= \discount (\IdMat - \discount \Pmat)^{-1} \Delta_{\Pmat} (\IdMat - \discount \Pmat)^{-1} 
		+ \discount^2 (\IdMat - \discount \PmatPrime)^{-1} \Delta_{\Pmat} (\IdMat - \discount \Pmat)^{-1} \Delta_{\Pmat}(\IdMat - \discount \Pmat)^{-1}.
\end{split}
\end{align*}
Substituting this identity into Eq.~\eqref{eqn:represent-theta-0-theta-1}, we obtain 
\begin{align}
\label{eqn:theta-diff-new-expression}
\begin{split}
\theta(\problem) - \theta(\problem^\prime) = \discount (\IdMat - \discount \Pmat)^{-1} \Delta_{\Pmat} \thetastar +  (\IdMat - \discount \Pmat)^{-1}\Delta_r 
	+ \remainder_{01},
\end{split}
\end{align}
where the remainder term $\remainder_{01}$ takes the form 
\begin{align*}
\remainder_{01} = \discount^2 (\IdMat - \discount \PmatPrime)^{-1}
\Delta_{\Pmat} (\IdMat - \discount \Pmat)^{-1} \Delta_{\Pmat}
\thetastar + \MatDiff{\Pmat}\Delta_r.
\end{align*}
Since $(1-\discount) (\IdMat -\discount \PmatPrime)^{-1}$ is a
probability transition matrix, it follows that $\norm{(1-\discount)
  (\IdMat -\discount \PmatPrime)^{-1}}_\infty \le 1$.  Thus, the
remainder term $\remainder_{01}$ satisfies the bound
\begin{align*}
\norm{\remainder_{01}}_\infty \le \frac{\discount}{(1-\discount)}
\norm{\Delta_{\Pmat}}_\infty \norm{\discount (\IdMat - \discount
  \Pmat)^{-1} \Delta_{\Pmat} \thetastar}_\infty +
\frac{\discount}{(1-\discount)^2}\norm{\Delta_{\Pmat}}_\infty\norm{\Delta_r}_\infty.
\end{align*}
The claimed lower bound~\eqref{eqn:theta-diff-lower-bound} now follows
from Eq.~\eqref{eqn:theta-diff-new-expression} and the triangle
inequality. It is clear that
Eq.~\eqref{eqn:theta-diff-lower-bound} implies the claim in the
lemma statement once we restrict $\problem^\prime \in \problemclass_1$
and $\problem^\prime \in \problemclass_2$.


\subsection{Proof of Lemma~\ref{lemma:proof-dhel-upper-bound}}
\label{sec:proof-lemma-proof-dhel-upper-bound}

Throughout the proof, we use $(\obsX, \obsr)$ (respectively
$(\obsX^\prime, \obsr^\prime)$) to denote a sample drawn from the
distribution $P$ (respectively from the distribution ${P^\prime}$). We
use $P_{\obsX}, P_{\obsr}$ (respectively $P^\prime_{\obsX},
P^\prime_{\obsr}$) to denote the marginal distribution of $\obsX,
\obsr$ (respectively $\obsX^\prime$, $\obsr^\prime$).  By the
independence of $\obsX$ and $\obsr$ (and similarly for $(\obsX^\prime,
\obsr^\prime)$, the joint distributions have the product form
\begin{align}
\label{eqn:independence-Z-r}
P = P_{\obsX} \otimes P_{\obsr}, \quad \mbox{and} \quad P^\prime =
P^\prime_{\obsX} \otimes P^\prime_{\obsr}.
\end{align}

\paragraph{Proof of part (a):}
Let $\problem^\prime = (\PmatPrime, \obsr^\prime) \in \problemclass_1$
(so $r^\prime = r$).  Because of the independence between $\obsX$ and
$\obsr$ (see Eq.~\eqref{eqn:independence-Z-r}) and $r =
r^\prime$, we have that
\begin{align*}
\dhel(P, {P^\prime})^2 = \dhel( P_{\obsX}, 
  P_{\obsX^\prime})^2.
\end{align*}
Note that the rows of $\obsX$ and $\obsX^\prime$ are independent.
Thus, if we let $\obsX_{i}, \obsX^\prime_{i}$ denote the $i$-th rows
of $\obsX$ and $\obsX^\prime$, we have
\begin{align*}
\dhel( P_{\obsX}, P_{\obsX^\prime})^2 = 1 - \prod_{i}
\left(1- \dhel( P_{\obsX_{i}}, P_{\obsX^\prime_{i}} )^2\right)
\le \sum_i \dhel( P_{\obsX_{i}}, P_{\obsX^\prime_{i}} )^2.
\end{align*}
Now, note that $\obsX_{i}$ and $\obsX^\prime_{i}$ have multinomial
distribution with parameters $\Pmat_{i}$ and $\PmatPrime_{i}$, where
$P_{0, i}, P_{1, i}$ are the $i$th row of $P_0$ and $P_1$. Thus, we
have
\begin{align*}
\dhel( P_{\obsX_{i}}, P_{\obsX^\prime_{i}})^2 \le \half
\dchi{P_{\obsX^\prime_{i}}}{P_{\obsX_{i}}} = \half \sum_{j}
\frac{\left(\Pmat_{i, j} - \PmatPrime_{i, j}\right)^2}{\Pmat_{i, j}}.
\end{align*}
Putting together the pieces yields the desired upper
bound~\eqref{eqn:dhel-upper-bound-one}.

\paragraph{Proof of part (b):} Let
$\problem^\prime = (\PmatPrime, \obsr^\prime) \in \problemclass_2$ (so
$\PmatPrime = \Pmat$).  Given the independence between $\obsX$ and
$\obsr$ (see Eq.~\eqref{eqn:independence-Z-r}) and $\Pmat =
\PmatPrime$, we have the relation $\dhel( P, {P^\prime} )^2 =
\dhel( P_{\obsr}, P_{\obsr^\prime} )^2$.  Note that $\obsr \sim
\normal(r, \IdMat)$ and $\obsr^\prime~\sim~\normal(r^\prime, \IdMat)$. Thus, we
have
\begin{align*}
 \dhel( P_{\obsr}, P_{\obsr^\prime} )^2 \le
 \dkl{P_{\obsr}}{P_{\obsr^\prime}} =
 \frac{1}{2\sigma_r^2}\ltwo{r-r^\prime}^2,
\end{align*}
as claimed.


\subsection{Proof of Lemma~\ref{lemma:construct-hardest-one-dimensional-alternatives}}
\label{sec:lemma-construct-hardest-one-dimensional-alternatives}

We now specify how to construct the probability matrix $\hardPmat$
that satisfies the desired properties stated in
Lemma~\ref{lemma:construct-hardest-one-dimensional-alternatives}.  We
introduce the shorthand notation $\bar{\theta} = \Pmat \thetastar$,
and $\Umat = (\IdMat - \discount \Pmat)^{-1}$. Let $\bar{\ell} \in
[D]$ be an index such that
\begin{align*}
\bar{\ell} \in \argmax_{\ell \in [D]}\left(e_{\ell}^{\top} (\IdMat -
\discount \Pmat)^{-1} \Sigma (\theta)(\IdMat - \discount \Pmat)^{-\top}
e_{\ell}\right)^{1/2} = \argmax_{\ell \in [D]}\Big( \sum_i
\Umat_{\ell, i}^2 \sigma^2_i(\thetastar)\Big)^{1/2}
\end{align*}
We construct the matrix $\hardPmat$ entrywise as follows:
\begin{align*}
\hardPmat_{i, j} = \Pmat_{i, j} + \frac{1}{\nu\sqrt{2\nobs}} \cdot
\Pmat_{i, j} \Umat_{\bar{\ell}, i} (\thetastar_j - \bar{\theta}_i)
\end{align*}
for $\nu \equiv \nu(\Pmat, \thetastar) = \Big( \sum_i
\Umat_{\bar{\ell}, i}^2 \sigma^2(\theta_i)\Big)^{1/2}$. Now we show
that $\hardPmat$ satisfy the following properties:
\begin{enumerate}
\item[(P1)] The matrix $\hardPmat$ is a probability transition
  matrix.
\item[(P2)] It satisfies the constraint $\sum_{i, j} \frac{\left(
  (\Pmat - \hardPmat)_{i, j} \right)^2}{\Pmat_{i, j}} \le
  \frac{1}{2\nobs}$.
\item[(P3)] It satisfies the inequalities
  \begin{align}
    \label{EqnPropertyThree}
\norm{\Pmat - \hardPmat}_\infty \le \frac{1}{\sqrt{2N}}, \quad
\mbox{and} \quad \norm{\discount (\IdMat - \discount \Pmat)^{-1}
  (\Pmat - \hardPmat) \thetastar}_\infty \ge
\frac{\discount}{\sqrt{2\nobs}} \cdot \nu(\Pmat, \thetastar).
  \end{align}
\end{enumerate}
\noindent We prove each of these properties in turn.


\paragraph{Proof of (P1):}
For each row $i \in [\Dim]$, we have
\begin{align}
  \sum_j \hardPmat_{i, j} = \sum_j \Pmat_{i, j} +
  \frac{1}{\nu\sqrt{2\nobs}} \Umat_{\bar{\ell}, i}
  \sum_j \Pmat_{i, j}(\thetastar_j - \bar{\theta}_{i}) =
  \sum_j \Pmat_{i, j} = 1,
\end{align}
thus showing that $\hardPmat \one = \one$ as desired.  Moreover,
since $(1-\discount)\Umat = (1-\discount) (\IdMat - \discount \Pmat)^{-1}$
is a probability transition matrix, we have
the bound
$|\Umat_{\bar{\ell}, i}| \le \frac{1}{1-\discount}$.
By the triangle inequality, we have
\begin{align*}
  2 \| \thetastar\|_{\myspan}
  \ge |\thetastar_j - \bar{\theta}_{i}|.
\end{align*}
Thus, our assumption on the sample size $\nobs$ implies that $\nu
\sqrt{\nobs} \ge \frac{2}{1-\discount}\| \thetastar\|_{\myspan} \ge
|\Umat_{\bar{l}, i} (\thetastar_j - \bar{\theta}_{j})|$, which further
implies that
\begin{align*}
  \hardPmat_{i, j} = \Pmat_{i, j}\left(1+\frac{1}{\nu\sqrt{2\nobs}}
  \cdot \Umat_{\bar{\ell}, i} (\thetastar_j - \bar{\theta}_i)\right)
  \ge 0.
\end{align*}
In conjunction with the property $\hardPmat \one = \one$, we conclude
that $\hardPmat$ is a probability transition matrix, as claimed.


\paragraph{Proof of (P2):}
We begin by observing that $(\Delta_{\Pmat})_{i, j} =
\frac{1}{\nu\sqrt{2\nobs}} \cdot \Pmat_{i, j} \Umat_{\bar{\ell}, i}
(\thetastar_j - \bar{\theta}_i)$.  Now it is simple to check that
\begin{align}
  \label{eqn:stupid-calc-one}
  \sum_{i, j} \frac{\left( (\Delta_{\Pmat})_{i, j}
    \right)^2}{\Pmat_{i, j}} = \frac{1}{2N \nu^2 }\sum_{i, j}
  \Pmat_{i, j} \Umat_{\bar{\ell}, i}^2 (\thetastar_j -
  \bar{\theta}_i)^2 \stackrel{(i)}{=} \frac{1}{2N \nu^2} \sum_i
  \Umat_{\bar{\ell}, i}^2 \sigma_i^2(\thetastar) = \frac{1}{2N},
\end{align}
where in step (i), we use $\sigma_i^2(\thetastar) = \sum_j \Pmat_{i,
  j}(\thetastar_j - \bar{\theta}_i)^2$ for each $i$, as the $i$th row
of our observation $\obsX$ is a multinomial distribution with mean
specified by the $i$th row of $\Pmat$.  This proves that $\hardPmat$
satisfies the constraint, as desired.

\paragraph{Proof of (P3):}
In order to verify the first inequality, we note that for any row $i$,
\begin{align*}
  \sum_{j} \left|(\Delta_{\Pmat})_{i, j}\right| \stackrel{(i)}{\le}
  \Big(\sum_{j} \frac{(\Delta_{\Pmat})_{i, j}^2}{\Pmat_{i,
      j}}\Big)^{1/2} \le \bigg(\sum_{i, j}\frac{(\Delta_{\Pmat})_{i,
      j}^2}{\Pmat_{i, j}}\bigg)^{1/2} \stackrel{(ii)}{=}
  \frac{1}{\sqrt{2\nobs}},
\end{align*}  
where step (i) follows from the Cauchy-Schwartz inequality, and step
(ii) follows by the previously established Property 2. Taking the
maximum over row $i$ yields
\begin{align*}
  \norm{\Delta_{\Pmat}}_\infty = \max_i \bigg\{\sum_{j}
  \left|(\Delta_{\Pmat})_{i, j}\right|\bigg\} \le
  \frac{1}{\sqrt{2\nobs}},
\end{align*}
thus establishing the first claimed inequality in
Eq.~\eqref{EqnPropertyThree}.

In order to establish the second inequality in
Eq.~\eqref{EqnPropertyThree}, our starting point is the lower
bound
\begin{align*}
  \norm{\discount (\IdMat - \discount \Pmat)^{-1}
    \Delta_{\Pmat} \thetastar}_\infty \ge
  \left|e_{\bar{\ell}}^T \discount (\IdMat - \discount
  \Pmat)^{-1} \Delta_{\Pmat} \thetastar\right| = \discount
  \cdot \Big|\sum_{i, j} \Umat_{\bar{\ell}, i}
  (\Delta_{\Pmat})_{i, j} \thetastar_j\Big|.
\end{align*}
It is straightforward to check that 
\begin{align*}
  \sum_{i, j} \Umat_{\bar{\ell}, i} (\Delta_{\Pmat})_{i,
    j} \thetastar_j \stackrel{(i)}{=} \sum_{i, j}
  \Umat_{\bar{\ell}, i} (\Delta_{\Pmat})_{i, j}
  (\thetastar_j - \bar{\theta}_i) = \frac{1}{\nu
    \sqrt{2N}}\sum_{i, j} \Pmat_{i, j}
  \Umat_{\bar{\ell}, i}^2 (\thetastar_j -
  \bar{\theta}_i)^2 \stackrel{(ii)}{=}
  \frac{\nu}{\sqrt{2N}}.
\end{align*}
Here step (i) follows from the fact that $\sum_j (\Delta_{\Pmat})_{i,
  j} = 0$ for all $i$ (as $\Delta_{\Pmat} \one = \hardPmat \one -
\Pmat\one = 0$); whereas step (ii) follows from our previous
calculation (see Eq.~\eqref{eqn:stupid-calc-one}) showing that
\begin{align*}
  \sum_{i, j} \Pmat_{i, j} \Umat_{\bar{\ell}, i}^2 (\thetastar_j -
  \bar{\theta}_i)^2 = \nu^2.
\end{align*}
Thus, we have verified the second inequality in
Eq.~\eqref{EqnPropertyThree}.


\section{Proofs of auxiliary lemmas for Theorem~\ref{thm:value_function_estimation}}

This appendix is devoted to the proofs of auxiliary lemmas involved in
the proof of Theorem~\ref{thm:value_function_estimation}.


\subsection{Proof of Lemma~\ref{lem:epoch-error-lemma}:}
\label{proof:lem:epoch-error-lemma}

In this section, we prove all three parts of
Lemma~\ref{lem:epoch-error-lemma}, which provides high-probability
upper bounds on the suboptimality gap at the end of each epoch.
Parts (a), (b) and (c), respectively, of
Lemma~\ref{lem:epoch-error-lemma} provides guarantees for the
recentered linear stepsize, polynomially-decaying stepsizes and
constant stepsize.  In order to de-clutter the notation, we omit the
dependence on the epoch $m$ in the operators and epoch initialization
$\thetabar_m$. In order to distinguish between the total sample size
$\nobs$ and the recentering sample size at epoch $m$, we retain
the notation $\nobs_m$ for the recentering sample size.


\subsubsection{Proof of part (a)}

We begin by rewriting the update
Eq,~\eqref{eqn:perturbed-update-equation} in a form suitable for
application of general results from~\cite{wainwright2019stochastic}.  Subtracting off the fixed point
$\thetahat$ of the operator $\Jstar$, we find that
\begin{align*}
\theta_{k+1} - \widehat{\theta} = ( 1 -\tdstep_{k} ) \left(\theta_{k} -
\widehat{\theta}\right) + \tdstep_{k} \left \{ \Jstar_{k}(\theta_{k})
- \widehat{\theta} \right \}.
\end{align*}
Note that the operator $\theta \mapsto \widehat{\Jstar}_{k}(\theta)$
is $\discount$-contractive in the $\ell_\infty$-norm and monotonic
with respect to the orthant ordering; consequently, Corollary 1 from~\cite{wainwright2019stochastic} can be applied. In applying
this corollary, the effective noise term is given by
\begin{align*}
W_{k} & \mydefn \Jstar_{k}( \widehat{\theta}) - \Jstar(\widehat{\theta})
\; = \; \left \{ \That_{k}(\widehat{\theta})-\That_{k}(\thetabar)
\right \} -\left \lbrace \Tstar(\widehat{\theta}) - \Tstar(\thetabar)
\right \rbrace.
\end{align*}
With this setup, by adapting Corollary 1 from~\cite{wainwright2019stochastic} we have
\begin{subequations}
\begin{align}
\label{eqn:SA-bound-linear}
  \left\|\theta_{\epochLen+1}-\widehat{\theta}\right\|_{\infty} \leq
  \frac{2}{1+(1-\discount)
    \epochLen}\left\{\|\thetabar-\widehat{\theta}\|_{\infty}+\sum_{k =
    1}^{\epochLen}\left \|
  \Aux_{\iter}\right\|_{\infty}\right\}+\left\|
  \Aux_{\epochLen+1}\right\|_{\infty},
\end{align}
where the auxiliary stochastic process $\{\Aux_k \}_{k \geq 1}$
evolves according to the recursion
\begin{align}
\label{EqnDefnAuxProcess}
\Aux_{k + 1} & = \big (1 - \tdstep_{k} \big) \Aux_{k} + \tdstep_{k}
W_{k}.
\end{align}
\end{subequations}
We claim that the $\ell_\infty$-norm of this process can be bounded
with high probability as follows:
\begin{lems}
\label{lem:Pl-bound}
Consider any sequence of stepsizes $\{ \tdstep_{k} \}_{k \geq 1}$ in
$(0,1)$ such that
\begin{align}
  \label{EqnStepCondition}
  (1 - \tdstep_{k + 1}) \tdstep_{k} \leq \tdstep_{k + 1}.
\end{align}
Then for any tolerance level $\delta > 0$, we have
\begin{align}
\label{EqnAuxBound}  
  \mathbb{P}\left[\left\|\Aux_{\iter + 1} \right\|_{\infty} \geq 4 \|
    \widehat{\theta} - \thetabar \|_{\infty} \sqrt{\tdstep_{\iter}}
    \sqrt{\log (8 K M D / \delta)}\right] \leq \frac{\delta}{2 K M}.
\end{align}
\end{lems}
\noindent See Appendix~\ref{proof:lem:Pl-bound} for a proof of this
claim. For future reference, note that all three stepsize
choices~\eqref{eqn:constant_step}--\eqref{eqn:recentered-linear}
satisfy the condition~\eqref{EqnStepCondition}. \\

\noindent Substituting the bound~\eqref{EqnAuxBound} into the
relation~\eqref{eqn:SA-bound-linear} yields
\begin{align*}
\left\|\theta_{\epochLen+1}-\thetahat\right\|_{\infty} & \leq
c\left\{\frac{\|\thetabar-\thetahat\|_{\infty}}{1+(1-\discount)
  \epochLen}+\frac{\|\thetabar-\thetahat\|_{\infty}}{(1-\discount)^{3
    / 2} \sqrt{\epochLen}}\right\} \sqrt{\log (8 \Klambda M D/
  \delta)} \\
& \leq c \|\thetabar-\thetahat\|_{\infty} \left\{\frac{\sqrt{\log (8
    \Klambda M D / \delta)}}{1+(1-\discount) \epochLen} +
\frac{\sqrt{\log (8 \Klambda M D / \delta)}}{(1-\discount)^{3 / 2}
  \sqrt{\epochLen}} \right\},
\end{align*}
with probability at least $1-\frac{\delta}{2 M}$. Combining the last
bound with the fact that $KM = \frac{\nobs}{2}$ we find that for all
$\epochLen \geq \unicon_1 \frac{\log(8 \nobs \Dim/ \delta)
}{(1-\discount)^{3}},$ we have
\begin{align*}
\|\theta_{\epochLen + 1} - \thetahat \|_{\infty} \leq
\tfrac{1}{\TwiceBaseSq} \|\thetabar - \thetahat\|_{\infty} & \leq
\tfrac{1}{\TwiceBaseSq} \|\thetabar - \thetastar \|_{\infty} +
\tfrac{1}{\TwiceBaseSq} \|\thetahat - \thetastar \|_{\infty},
\end{align*}
which completes the proof of part (a).


\subsubsection*{Proof of part (b):}

The proof of part (b) is similar to that of part (a).  In particular,
adapting Corollary 2 from the paper~\cite{wainwright2019stochastic}
for polynomial steps, we have
\begin{align}
\label{eqn:poly-bound-intermediate}
\|\theta_{k+1}- \widehat{\theta} \|_{\infty} \leq
e^{-\frac{1-\discount}{1-\omega}\left(k^{1-\omega}-1\right)} \|
\widebar{\theta} - \widehat{\theta} \|_{\infty} +
e^{-\frac{1-\discount}{1-\omega} k^{1-\omega}} \sum_{\ell=1}^{k}
\frac{e^{\frac{1-\discount}{1-\omega} \ell^{1-\omega}}}{\ell^{\omega}}
\|\Aux_{\ell} \|_{\infty} + \| \Aux_{k+1} \|_{\infty}.
\end{align}
Recall that polynomial stepsize~\eqref{eqn:poly_step} satisfies the
conditions of Lemma~\ref{lem:Pl-bound}. Consequently, applying the
bound from Lemma~\ref{lem:Pl-bound} we find that
\begin{align}
\|\theta_{k+1}- \widehat{\theta} \|_{\infty} & \leq \| \thetabar -
\widehat{\theta} \|_{\infty} \left \{
e^{-\frac{1-\discount}{1-\omega}\left(k^{1-\omega}-1\right)} + 4
\sqrt{\log (8 K M D / \delta)} \left( e^{-\frac{1-\discount}{1-\omega}
  k^{1-\omega}} \sum_{\ell=1}^{k}
\frac{e^{\frac{1-\discount}{1-\omega} \ell^{1-\omega}}}{\ell^{3 \omega
    / 2 }} + \frac{1}{k^{\omega/2}} \right) \right \}.
\end{align} 

It remains to bound the coefficient of $\| \thetabar - \thetahat
\|_{\infty}$ in the last equation, and we do so by using the following
lemma from~\cite{wainwright2019stochastic}:
\begin{lems}[Bounds on exponential-weighted sums]
There is a universal constant $c$ such that for all $\omega \in(0,1)$
and for all $k \geq\left(\frac{3
  \omega}{2(1-\discount)}\right)^{\frac{1}{1-\omega}},$ we have
\begin{align*}
e^{-\frac{1-\discount}{1-\omega} k^{1-\omega}} \sum_{\ell=1}^{k}
\frac{e^{\frac{1-\discount}{1-\omega} \ell^{1-\omega}}}{\ell^{3 \omega
    / 2}} & \leq c \left \{ \frac{e^{-\frac{1-\discount}{1-\omega}
    (k^{1-\omega})}}{(1-\discount)^{\frac{1}{1-\omega}}} +
\frac{1}{(1-\discount)} \frac{1}{k^{\omega / 2}} \right \}.
\end{align*}
\end{lems}
\vspace{10pt}

\noindent 
Substituting the last bound in
Eq.~\eqref{eqn:poly-bound-intermediate} yields
\begin{align*}
  \left\|\theta_{k+1}- \widehat{\theta} \right\|_{\infty}
  & \leq \unicon \| \thetabar - \widehat{\theta} \|_{\infty}
  \left\lbrace e^{-\frac{1-\discount}{1-\omega}\left(k^{1-\omega}-1\right)}
  + 4 \sqrt{\log (8 K M D / \delta)}  \left( \frac{e^{-\frac{1-\discount}{1-\omega}\left(k^{1-\omega}-1\right)}}{(1-\discount)^{\frac{1}{1-\omega}}} + \frac{1}{(1-\discount)} \frac{1}{k^{\omega / 2}}  + \frac{1}{k^{\omega/2}} \right)  \right\rbrace \\
  & \leq   \unicon \| \thetabar - \widehat{\theta} \|_{\infty}
  \cdot \sqrt{\log (8 K M D / \delta)} 
  \left\lbrace  5 \cdot \frac{e^{-\frac{1-\discount}{1-\omega} \left(k^{1-\omega}-1\right)}}{(1-\discount)^{\frac{1}{1-\omega}}} + \frac{2}{(1 - \discount) k^{\omega/2}} \right\rbrace .
\end{align*}
Finally, doing some algebra and using the fact that $KM =
\frac{\nobs}{2}$ we find that there is an absolute constant $\unicon$
such that for all $\Klambda$ lower bounded as \mbox{$\Klambda \geq c
  \log (4 \nobs \Dim / \delta) \cdot \left( \frac{1}{1 - \discount}
  \right)^{\frac{1}{1 - \omega} \vee \frac{2}{\omega}}$,} we have
\begin{align*}
\| \theta_{\Klambda + 1} - \thetahat \|_\infty \leq \frac{\| \thetabar
  - \thetahat \|_{\infty}}{\TwiceBaseSq} \leq \frac{1}{\TwiceBaseSq}
\|\thetabar - \thetastar \|_{\infty} + \frac{1}{\TwiceBaseSq}
\|\thetahat - \thetastar \|_{\infty}.
\end{align*}
The completes the proof of part (b).


\subsubsection*{Proof of part (c):}

Invoking Theorem 1 from~\cite{wainwright2019stochastic}, we
have $\| \theta_{\Klambda} - \widehat{\theta} \|_{\infty} \leq
a_{\Klambda} + b_{\Klambda} + \| \Aux_{\Klambda} \|_{\infty}$.  For a
constant stepsize $\tdstep_{k} = \constStep$, the pair $(a_{\Klambda},
b_{\Klambda})$ is given by
\begin{align*}
  b_{\Klambda} & = \left\| \thetabar - \widehat{\theta}
  \right\|_{\infty} \cdot \left(1 - \constStep (1-\discount)
  \right)^{\Klambda - 1}, \\ a_{\Klambda} & = \discount \constStep
  \left\| \Aux_{k}\right\|_{\infty} + \discount \constStep \left\|
  \Aux_{\ell} \right\|_{\infty} \sum_{k=1}^{\Klambda - 1} \left\{
  \left(1-(1-\discount) \constStep \right)^{\Klambda - k} \right\}
  \\ & \stackrel{(i)}{\leq} \| \thetabar - \widehat{\theta}
  \|_{\infty} \cdot \left( 2 \discount \constStep^{\frac{3}{2}}
  \sqrt{\log (8 \Klambda \TotalEpochs \Dim / \delta)} + \frac{2 \gamma
    {\constStep}^{ \frac{1}{2}}}{1 - \gamma} \sqrt{\log (8 \Klambda
    \TotalEpochs \Dim / \delta)} \right),
\end{align*}
where inequality~(i) follows by substituting $\tdstep_{k} =
\constStep$, and using the bound on $\| \Aux_{\ell} \|_{\infty}$ from
Lemma~\ref{lem:Pl-bound}.

It remains to choose the pair $(\constStep, \Klambda)$ such that $\|
\theta_{\Klambda + 1} - \widehat{\theta} \|_{\infty} \leq
\tfrac{1}{\TwiceBaseSq} \| \thetabar - \widehat{\theta} \|_{\infty}$.
Doing some simple algebra and using the fact that $\Klambda
\TotalEpochs = \frac{\nobs}{2}$ we find that it is sufficient to
choose the pair $(\constStep, \Klambda)$ satisfying the conditions
\begin{align*}
  0 < \constStep \leq \frac{(1 - \discount)^2}{ \log \left( 4 \nobs
    \Dim / \delta \right)} \cdot \frac{1}{5^2 \cdot 32^2}, \quad
  \text{and} \quad \Klambda \geq 1 + \frac{2 \log 16}{\log \left(
    \frac{1}{1 - \tdstep(1 - \discount)} \right)}.
\end{align*}
With this choice, we have
\begin{align*}
  \| \theta_{\Klambda + 1} - \thetahat \|_\infty \leq \frac{\|
    \thetabar - \thetahat \|_{\infty}}{\TwiceBaseSq} \leq
  \frac{1}{\TwiceBaseSq} \left \| \thetabar - \thetastar
  \right\|_{\infty} + \frac{1}{\TwiceBaseSq} \left \| \thetahat -
  \thetastar \right\|_{\infty},
\end{align*}
which completes the proof of part (c).


\subsection{Proof of Lemma~\ref{lem:epoch-deviation-lemma}}
\label{proof:lem:epoch-deviation-lemma}

Recall our shorthand notation for the local
complexities~\eqref{eq:local-complexity}. The following lemma
characterizes the behavior of various random variables as a function
of these complexities.  In stating the lemma, we let $\Phat_{n}$ be a
sample transition matrix constructed as the average of $n$
i.i.d. samples, and let $\rhat_n$ denote the reward vector
constructed as the average of $n$ i.i.d. samples.

\begin{lemma}
\label{lem:bern-hoeff}
Each of the following statements holds with probability exceeding $1 -
\frac{\delta}{\TotalEpochs}$:
\begin{align*}
\| (\IdMat - \discount \Pmat)^{-1} (\Phat_n - \Pmat) \thetastar
\|_{\infty} &\leq 2 \nu(\Pmat, \thetastar) \cdot \sqrt{ \frac{\log (4
    \Dim \TotalEpochs/ \delta)}{n}} + 4 \cdot b(\thetastar) \cdot
\frac{\log (4 \Dim \TotalEpochs/ \delta)}{n}, \text{ and } \\ \|
(\IdMat - \discount \Pmat)^{-1} (\rhat_n - r) \|_{\infty} &\leq 2
\rho(\Pmat, r) \cdot \sqrt{ \frac{\log (4 \Dim \TotalEpochs /
    \delta)}{n}}.
\end{align*}
\end{lemma}
\begin{proof}
Entry $\ell$ of the vector $(\IdMat - \discount \Pmat)^{-1} (\Phat_n -
\Pmat) \thetastar$ is zero mean with variance given by the $\ell^{th}$
diagonal entry of the matrix $(\IdMat - \discount \Pmat)^{-1}
\Sigma(\thetastar) (\IdMat - \discount \Pmat)^{-T}$, and is bounded by
$b(\thetastar)$ almost surely. Consequently, applying the Bernstein
bound in conjunction with the union bound completes the proof of the
first claim.
In order to establish the second claim, note that the vector $(\IdMat
- \discount \Pmat)^{-1} (\rhat_n - r)$ has sub-Gaussian entries, and
apply the Hoeffding bound in conjunction with the union bound.
\end{proof}

In light of Lemma~\ref{lem:bern-hoeff}, note that it suffices to
establish the inequality
\begin{align}
\label{eq:equiv-bd}
\Pr \left\{ \| \thetahat_m - \thetastar \|_{\infty} \geq
\frac{\left\|\thetabar-\thetastar\right\|_{\infty}}{9} + \| (\IdMat -
\discount \Pmat)^{-1} (\Phat_{\Nm} - \Pmat) \thetastar \|_{\infty} +
\| (\IdMat - \discount \Pmat)^{-1} (\rhat_{\Nm} - r) \|_{\infty}
\right\} \leq \frac{\delta}{2 \TotalEpochs},
\end{align}
where we have let $\Phat_{\Nm}$ and $\rhat_{\Nm}$ denote the empirical
mean of the observed transitions and rewards in epoch $m$,
respectively.  The proof of Lemma~\ref{lem:epoch-deviation-lemma}
follows from Eq.~\eqref{eq:equiv-bd} by a union bound.


\paragraph{Establishing the bound~\eqref{eq:equiv-bd}:}

Since the epoch number $m$ should be clear from context, let us adopt
the shorthand $\thetahat \equiv \thetahat_m$, along with the shorthand
$\rhat \equiv \rhat_{\Nm}$ and $\Phat \equiv \Phat_{\Nm}$.  Note that
$\thetahat$ is the fixed point of the following operator:
\begin{align*}
  \Jstar(\theta) & \mydefn
  \Tstar(\theta)-\Tstar(\thetabar)+\Ttil_{\Nm}(\thetabar) =
  \underbrace{{\widehat{r} + \discount \left( \widehat{\Pmat} - \Pmat
      \right) \thetabar}}_{\widetilde{r}} + \discount \Pmat\theta,
\end{align*}
where we have used the fact that $\widetilde{T}_{\Nm}(\theta) =
\widehat{r} + \discount \widehat{\Pmat} \theta$.

Thus, we have $\thetahat = (\IdMat - \discount \Pmat)^{-1} \rtil$, so
that $\thetahat - \thetastar = (\IdMat - \discount \Pmat )^{-1}
\left(\rtil - r\right)$.  Also note that we have
\begin{align*}
\widetilde{r}-r = {\widehat{r} + \discount \left( \widehat{\Pmat} -
  \Pmat \right) \thetabar} - r &= \widehat{r} - r + \discount \left(
\widehat{\Pmat} - \Pmat \right) \thetastar + \discount \left(
\widehat{\Pmat} - \Pmat \right) \left( \thetabar - \thetastar \right),
\end{align*}
so that putting together the pieces and using the triangle inequality
yields the bound
\begin{align*}
\| \thetahat - \thetastar \|_{\infty} &\leq \| (\IdMat - \discount
\Pmat)^{-1} (\widehat{r} - r) \|_{\infty} + \discount \| (\IdMat -
\discount \Pmat)^{-1} \left( \widehat{\Pmat} - \Pmat \right)
\thetastar \|_{\infty} + \discount \|(\IdMat - \discount \Pmat)^{-1}
\left( \widehat{\Pmat} - \Pmat \right) \left( \thetabar - \thetastar
\right) \|_{\infty} \\
& \leq \| (\IdMat - \discount \Pmat)^{-1} (\widehat{r} - r)
\|_{\infty} + \discount \| (\IdMat - \discount \Pmat)^{-1} \left(
\widehat{\Pmat} - \Pmat \right) \thetastar \|_{\infty} +
\frac{\discount}{1 - \discount} \| \left( \widehat{\Pmat} - \Pmat
\right) \left( \thetabar - \thetastar \right) \|_{\infty}.
\end{align*}
Note that the random vector $\left( \widehat{\Pmat} - \Pmat \right)
\left( \thetabar - \thetastar \right)$ is the empirical average of
$\Nm$ i.i.d. random vectors, each of which is bounded entrywise by $2
\| \thetabar - \thetastar \|_\infty$. Consequently, by a combination
of Hoeffding's inequality and the union bound, we find that
\begin{align*}
\left\| \left( \widehat{\Pmat} - \Pmat \right) \left( \thetabar -
\thetastar \right) \right\|_\infty \leq 4 \| \thetabar - \thetastar
\|_\infty \sqrt{\LOGterm{m}},
\end{align*}
with probability at least $1-\frac{\delta}{4 M}$. Thus, provided $\Nm \geq 4^2 \cdot {\TwiceBaseSqPlusOne}^2 \cdot \frac{\discount^2}{(1 - \discount)^2} \log (8\Dim M /\delta)$ for a
large enough constant $\unicon_1$, we have
\begin{align*}
\frac{\discount}{1 - \discount} \left\| \left( \widehat{\Pmat} - \Pmat \right) \left( \thetabar - \thetastar \right)  \right\|_\infty \leq \frac{\| \thetabar - \thetastar \|_\infty}{\TwiceBaseSqPlusOne}.
\end{align*}
This completes the proof.

\subsection{Proof of Lemma~\ref{lem:Pl-bound}}

\label{proof:lem:Pl-bound}
Recall that by definition, the stochastic process $\{\Aux_k \}_{k \geq
  1}$ evolves according to the linear recursion \mbox{$\Aux_k = (1 -
  \tdstep_{k}) \Aux_{k -1} + \tdstep_{k} \W_{k - 1}$,} where the
effective noise sequence $\{\W_{k} \}_{k \geq 0}$ satisfies the
uniform bound
\begin{align*}
  \|\W_{k}\|_{\infty} \leq \left\|
  \That_{k}(\widehat{\theta})-\That_{k}(\thetabar)\right\|_{\infty} +
  \|\Tstar(\widehat{\theta})-\Tstar(\thetabar)\|_\infty \leq
  \underbrace {2 \| \widehat{\theta} - \thetabar \|_{\infty}}_{:= b}
  \quad \text{for all } k \geq 0.
\end{align*}
Moreover, we have $\Exs[\W_{k}] = 0$ by construction so that each
entry of the random vector $\W_{k}$ is a zero-mean sub-Gaussian random
variable with sub-Gaussian parameter at most $2 \| \widehat{\theta} -
\thetabar \|_{\infty}$. Consequently, by known properties of
sub-Gaussian random variables (cf.\ Chapter 2 in~\cite{wainwright2019high}), we have
\begin{align}
\label{eqn:mgf-bound-Wk}
\log \Exs \left[ e^{s \W_{k}(x)} \right] \leq \frac{s^2 b^2}{8} \qquad
\mbox{for all scalars $s \in \real$, and states $x$.}
\end{align}

We complete the proof by using an inductive argument to upper bound
the moment generating function of the random variable $\Aux_{\ell}$;
given this inequality, we can then apply the Chernoff bound to obtain
the stated tail bounds.  Beginning with the bound on the moment
generating function, we claim that
\begin{align}
  \label{eqn:mgf-bound-Pk}
  \log \Exs \big[ e^{s \Aux_{k}(x)} \big] \leq \tfrac{s^2 \tdstep_{k}
    b^2 }{8} \qquad \mbox{for all scalars $s \in \real$ and states
    $x$.}
\end{align}
We prove this claim via induction on $k$.

\paragraph{Base case:} For $k = 1$, we have
\begin{align*}
  \log \Exs \left[ e^{s \Aux_{1}(x)} \right] & = \log \Exs \left[ e^{s
      \tdstep_{1} \W_{0}(x)} \right] \leq \tfrac{s^2 \tdstep_{1}^2
    b^2}{8},
\end{align*}
where the first equality follows from the definition of
$\Aux_1$, and the second inequality follows by applying the
bound~\eqref{eqn:mgf-bound-Wk}.

\paragraph{Inductive step:} We now assume that the bound~\eqref{eqn:mgf-bound-Pk} holds for some 
iteration $k \geq 1$ and prove that it holds for iteration $k +
1$. Recalling the definition of $\Aux_k$, and the independence of the
random variables $\Aux_k$ and $\W_{k}$, we have
\begin{align*}
\Exs \left[e^{s \Aux_{k+1}(x)}\right] & = \log \Exs\left[e^{s
    \left(1-\tdstep_{k}\right) \Aux_{k} (x)} \right] + \log
\Exs\left[e^{s \tdstep_{k} \W_{k}(x)} \right] \\ 
& \leq \tfrac{s^{2} \left(1-\tdstep_{k}\right)^{2} \tdstep_{k-1}
  b^{2}}{8} + \tfrac{s^{2} \tdstep_{k}^{2} b^{2}}{8} \\
& \stackrel{(i)}{\leq} \tfrac{s^{2} \left(1-\tdstep_{k}\right)
  \tdstep_{k} b^{2} }{8} + \tfrac{s^{2} \tdstep_{k}^{2} b^{2}}{8} \\
& = \tfrac{s^{2} \tdstep_{k} b^{2}}{8},
\end{align*}  
where inequality (i) follows from the assumed
condition~\eqref{EqnStepCondition} on the stepsizes.

Simple algebra yields that all the stepsize
choices~\eqref{eqn:constant_step}--\eqref{eqn:recentered-linear}
satisfy the condition~\eqref{EqnStepCondition}.  Finally, combining
the bound~\eqref{eqn:mgf-bound-Pk} with the Chernoff bounding technique
along with a union bound over iterations $k = 1, \ldots \Klambda$
yields
\begin{align*}
  \Prob \left[\left\| \Aux_{\ell} \right\|_{\infty} \geq 2 b
    \sqrt{\tdstep_{\ell-1}} \sqrt{\log (8 K M D / \delta)}\right] \leq
  \frac{\delta}{8 \Klambda \TotalEpochs},
\end{align*}
as claimed.


\end{document}

%% file: BearMacros-arXiv.tex
\theoremstyle{definition}

\newtheorem{lems}{Lemma}

\newlength{\widebarargwidth}
\newlength{\widebarargheight}
\newlength{\widebarargdepth}
\DeclareRobustCommand{\widebar}[1]{%
  \settowidth{\widebarargwidth}{\ensuremath{#1}}%
  \settoheight{\widebarargheight}{\ensuremath{#1}}%
  \settodepth{\widebarargdepth}{\ensuremath{#1}}%
  \addtolength{\widebarargwidth}{-0.3\widebarargheight}%
  \addtolength{\widebarargwidth}{-0.3\widebarargdepth}%
  \makebox[0pt][l]{\hspace{0.3\widebarargheight}%
    \hspace{0.3\widebarargdepth}%
    \addtolength{\widebarargheight}{0.3ex}%
    \rule[\widebarargheight]{0.95\widebarargwidth}{0.1ex}}%
  {#1}}

\makeatletter
\long\def\@makecaption#1#2{
        \vskip 0.8ex
        \setbox\@tempboxa\hbox{\small {\bf #1:} #2}
        \parindent 1.5em  
        \dimen0=\hsize
        \advance\dimen0 by -3em
        \ifdim \wd\@tempboxa >\dimen0
                \hbox to \hsize{
                        \parindent 0em
                        \hfil
                        \parbox{\dimen0}{\def\baselinestretch{0.96}\small
                                {\bf #1.} #2
                                }
                        \hfil}
        \else \hbox to \hsize{\hfil \box\@tempboxa \hfil}
        \fi
        }
\makeatother

\long\def\comment#1{}

\newcommand{\red}[1]{\textcolor{red}{#1}}

\newcommand{\Exs}{\ensuremath{{\mathbb{E}}}}
\newcommand{\Prob}{\ensuremath{{\mathbb{P}}}}

\newcommand{\usedim}{\ensuremath{D}}

\DeclareMathOperator{\diag}{diag}
\DeclareMathOperator{\var}{var}
\DeclareMathOperator{\cov}{cov}

\DeclareMathOperator{\myspan}{span}

\newcommand{\thetastar}{\ensuremath{\theta^*}}
\newcommand{\thetahat}{\ensuremath{\widehat{\theta}}}
\newcommand{\thetatil}{\ensuremath{\widetilde{\theta}}}
\newcommand{\thetabar}{\ensuremath{\widebar{\theta}}}

\newcommand{\widgraph}[2]{\includegraphics[keepaspectratio,width=#1]{#2}}

\newcommand{\grad}[1]{\ensuremath{\nabla #1}}

\newcommand{\mydefn}{\ensuremath{: \, =}}

\newcommand{\real}{\ensuremath{\mathbb{R}}}

\newcommand{\order}{\ensuremath{\mathcal{O}}}

\newcommand{\Pmat}{\ensuremath{\mathbf{P}}}
\newcommand{\Phat}{\ensuremath{\widehat{\mathbf{P}}}}
\newcommand{\IdMat}{\ensuremath{\mathbf{I}}}

\newcommand{\Tstar}{\ensuremath{\mathcal{T}}}
\newcommand{\That}{\ensuremath{\widehat{\Tstar}}}
\newcommand{\Ttil}{\ensuremath{\widetilde{\Tstar}}}
\newcommand{\Jstar}{\ensuremath{\mathcal{J}}}
\newcommand{\V}{\ensuremath{\mathcal{V}}}

\newcommand{\iter}{\ensuremath{k}}
\newcommand{\TotalEpochs}{\ensuremath{M}}
\newcommand{\Avglength}{\ensuremath{N}}
\newcommand{\nobs}{\ensuremath{N}}

\newcommand{\epochNum}{\ensuremath{m}}
\newcommand{\Nm}{\ensuremath{\Avglength_{\epochNum}}}
\newcommand{\Dm}{\ensuremath{\mathfrak{D}_\epochNum}}
\newcommand{\Em}{\ensuremath{\mathfrak{E}_\epochNum}}

\newcommand{\pardelta}{\ensuremath{\delta}}
\newcommand{\unicon}{\ensuremath{c}}
\newcommand{\iterM}{\ensuremath{\epochNum}}
\newcommand{\NM}{\ensuremath{N_M}}

\newcommand{\stepsize}{\ensuremath{\lambda}}
\newcommand{\epochLen}{\ensuremath{{K}}}
\newcommand{\Klambda}{\ensuremath{\epochLen}}

\newcommand{\rhat}{\ensuremath{\widehat{r}}}

\newcommand{\LOGDMN}{\ensuremath{ \frac{\log (8D M/ \delta)}{N} }}
\newcommand{\LOGDMNSQ}{\ensuremath{ \sqrt{\frac{\log (8D M / \delta)}{N} }}}

\newcommand{\LOGDMNm}{\ensuremath{ \frac{\log (8DM / \delta)}{N_m} }}
\newcommand{\LOGterm}[1]{\ensuremath{ \frac{\log (8DM / \delta)}{N_{#1} } }}

\newenvironment{carlist}
 {\begin{list}{$\bullet$}
 {\setlength{\topsep}{0in} \setlength{\partopsep}{0in}
  \setlength{\parsep}{0in} \setlength{\itemsep}{\parskip}
  \setlength{\leftmargin}{0.15in} \setlength{\rightmargin}{0.08in}
  \setlength{\listparindent}{0in} \setlength{\labelwidth}{0.08in}
  \setlength{\labelsep}{0.1in} \setlength{\itemindent}{0in}}}
 {\end{list}}

 \newcommand{\bcar}{\begin{carlist}}
\newcommand{\ecar}{\end{carlist}}

\newcommand{\W}{\ensuremath{W}}
\newcommand{\constStep}{\ensuremath{\tdstep}}

\newcommand{\base}{\ensuremath{2}}
\newcommand{\baseSq}{\ensuremath{4}}
\newcommand{\baseSqrt}{\ensuremath{\sqrt{2}}}
\newcommand{\TwiceBaseSq}{\ensuremath{8}}
\newcommand{\TwiceBaseSqPlusOne}{\ensuremath{9}}
\newcommand{\EpochSamples}{\ensuremath{\{ \That_{i} \}_{i \in \Em}}}